%% file: paper_arxiv_v2.tex
\documentclass[10pt,twocolumn,letterpaper]{article}

\input{math_commands.tex}

\usepackage{authblk}

\makeatletter
\renewcommand\AB@affilsepx{, \protect\Affilfont}
\makeatother

\usepackage[pagenumbers]{cvpr} 

\usepackage{graphicx}
\usepackage{amsmath}
\usepackage{amssymb}
\usepackage{booktabs}
\usepackage{lipsum}
\usepackage{xr}

\makeatletter
\@namedef{ver@everyshi.sty}{}
\makeatother
\usepackage{tikz-cd}

\usepackage[pagebackref,breaklinks,colorlinks]{hyperref}

\usepackage[capitalize]{cleveref}
\crefname{section}{Sec.}{Secs.}
\Crefname{section}{Section}{Sections}
\Crefname{table}{Table}{Tables}
\crefname{table}{Tab.}{Tabs.}

\def\cvprPaperID{5431} 
\def\confName{CVPR}
\def\confYear{2022}

\begin{document}

\title{Frame Averaging for Equivariant Shape Space Learning }

\author[1,2]{Matan Atzmon \thanks{Work done during an internship at NVIDIA.} }
\author[1]{Koki Nagano}
\author[1,3,4]{Sanja Fidler}
\author[1]{Sameh Khamis}
\author[2]{Yaron Lipman}
\affil[1]{NVIDIA}
\affil[2]{Weizmann Institute of Science}
\affil[3]{University of Toronto}
\affil[4]{Vector Institute}

\maketitle

\begin{abstract}

    The task of shape space learning involves mapping a train set of shapes to and from a latent representation space with good generalization properties. Often, real-world collections of shapes have symmetries, which can be defined as transformations that do not change the essence of the shape. A natural way to incorporate symmetries in shape space learning is to ask that the mapping to the shape space (encoder) and mapping from the shape space (decoder) are equivariant to the relevant symmetries. 
    In this paper, we present a framework for incorporating equivariance in encoders and decoders by introducing two contributions: (i) adapting the recent Frame Averaging (FA) framework for building generic, efficient, and maximally expressive Equivariant autoencoders;  and (ii) constructing autoencoders equivariant to piecewise Euclidean motions applied to different parts of the shape. To the best of our knowledge, this is the first fully piecewise Euclidean equivariant autoencoder construction. Training our framework is simple: it uses standard reconstruction losses, and does not require the introduction of new losses. Our architectures are built of standard (backbone) architectures with the appropriate frame averaging to make them equivariant. Testing our framework on both rigid shapes dataset using implicit neural representations, and articulated shape datasets using mesh-based neural networks show state of the art generalization to unseen test shapes, improving relevant baselines by a large margin. In particular, our method demonstrates significant improvement in generalizing to unseen articulated poses.

    \end{abstract}
    
    \section{Introduction}
    
    
    Learning a shape space is the task of finding a latent representation to a collection of input training shapes that generalizes well to unseen, test shapes. This is often done within an autoencoder framework, namely an \emph{encoder} $\Phi:X\too Z$, mapping an input shape in $X$ (in some 3D representation) to the latent space $Z$, and a \emph{decoder} $\Psi:Z\too Y$, mapping latent representations in $Z$ back to shapes $Y$ (possibly in other 3D representation than  $X$). 
    
    Many shape collections exhibit \emph{symmetries}. That is, transformations that do not change the essence of the shape. For example, applying an Euclidean motion (rotation, reflection, and/or translation) to a rigid object such as a piece of furniture will produce an equivalent version of the object. Similarly, the same articulated body, such as an animal or a human, can assume different poses in space. 
    
    A natural way to incorporate symmetries in shape space learning is to require the mapping to the latent space, i.e., the encoder, and mapping from the latent space, i.e., the decoder, to be equivariant to the relevant symmetries. That is, applying the symmetry to an input shape and then encoding it would result in the same symmetry applied to the latent code of the original shape. Similarly, reconstructing a shape from a transformed latent code will result in a transformed shape. 
    
    The main benefit in imposing equivariance in shape space learning is achieving a very useful inductive bias: If the model have learned a single shape, it can already generalize perfectly to all its symmetric versions! 
    Unfortunately, even in the presumably simpler setting of a global Euclidean motion, building an equivariant neural network that is both expressive and efficient remains a challenge. The only architectures that were known to be universal for Euclidean motion equivariant functions are Tensor Field Networks \cite{thomas2018tensor,dym2020universality} and group averaging \cite{yarotsky2021universal,chen2021equivariant} both are computationally and memory intensive. Other architectures, \eg, Vector Neurons \cite{deng2021vector} are efficient computationally but are not known to be universal. 
    
    In this paper, we present a novel framework for building equivariant encoders and decoders for shape space learning that are flexible, efficient and maximally expressive (\ie, universal). In particular, we introduce two contributions: (i) we adapt the recent Frame Averaging (FA) framework \cite{puny2021frame} to shape space learning, showing how to efficiently build powerful shape autoencoders. The method is general, easily adapted to different architectures and tasks, and its training only uses standard autoencoder reconstruction losses without requiring the introduction of new losses. (ii) We construct what we believe is the first autoencoder architecture that is fully equivariant to piecewise Euclidean transformations of the shape's parts, \eg, articulated human body. 
    
    We have tested our framework on two types of shape space learning tasks: learning implicit representations of shapes from real-life input point clouds extracted from sequences of images \cite{reizenstein2021common}, and learning mesh deformations of human (body and hand) and animal shape spaces \cite{dfaust:CVPR:2017,zuffi20173d,mahmood2019amass,akhter2015pose}. In both tasks, our method produced state of the art results when compared to relevant baselines, often showing a big margin compared to the runner-up, justifying the efficacy of the inductive bias injected using the frame-averaging and equivariance.  
    
    
    
    \section{Related work}
    \paragraph{Euclidean equivariant point networks.}
    Original point cloud networks, such as PointNet \cite{qi2017pointnet,qi2017pointnet++}, PCNN \cite{atzmon2018point}, PointCNN \cite{li2018pointcnn}, Spider-CNN \cite{xu2018spidercnn}, and DGCNN \cite{wang2019dynamic} are permutation equivariant but not Euclidean equivariant. Therefore, these architectures often struggle to generalize over translated and/or rotated inputs. Realizing that Euclidean equivariance is a useful inductive bias, much attention has been given to develop Euclidean equivariant point cloud networks. 
    Euclidean \emph{invariance} can be achieved by defining the network layers in terms of distances or angles  between points ~\cite{deng2018ppf,zhang2019rotation} or   angles and distances measured from the input point cloud's normals \cite{gojcic2019perfect}.
    Other works encode local neighborhoods using some local or global coordinate system to achieve invariance to rotations and translations. \cite{xiao2020endowing,yu2020deep,deng2018ppf} use PCA to define rotation invariance. 
    %
    %
    %
    %
    Equivariance is a desirable property for autoencoders. 
    Some works use representation theory of the rotation group (\eg, spherical harmonics) to build rotational equivariant networks \cite{worrall2017harmonic,liu2018deep,weiler20183d}. Tensor Field Networks (TFN) \cite{thomas2018tensor,fuchs2020se,romero2020group} achieve equivariance to both translation and rotation. However, TFN architectures are tailored to rotations and require high order features for universality  \cite{dym2020universality}. 
    Recently~\cite{deng2021vector} proposed a rotation equivariant network encoding features using the first two irreducible representations of the rotation group (tensor features) and constructed linear equivariant layers between features as well as equivariant non-linearities. This architecture is not proven universal. 
    Another method achieving Euclidean equivariance is by group averaging or convolutions \cite{yarotsky2021universal}. \cite{esteves2018learning,cohen2018spherical} use spherical convolution to achieve rotation or Euclidean equivariance. \cite{chen2021equivariant} suggests to average of the 6D Euclidean group. Recently, \cite{puny2021frame} suggest Frame Averaging (FA) as a general purpose methodology for building equivariant architectures that are maximally expressive and often offer a much more efficient computation than group representation or averaging techniques.
    
    \vspace{-5pt}
    \paragraph{Implicit shape space learning.}
    Learning neural implicit representations from input point clouds is done by regressing signed distance function to the surface \cite{park2019deepsdf} or occupancy probabilities \cite{chen2019learning,mescheder2019occupancy}. The input point cloud is usually encoded in the latent space using a PointNet-like encoder \cite{qi2017pointnet,zaheer2017deep} or autodecoder \cite{park2019deepsdf}.  \cite{atzmon2020sal,atzmon2020sald} regress the unsigned distance to the input point clouds avoiding the need of implicit function supervision for training. Normal data and gradient losses can be used to improve training and fidelity of learned implicits \cite{gropp2020implicit,sitzmann2020implicit,atzmon2020sald,lipman2021phase}. Higher spatial resolution was achieved by using spatially varying latent codes \cite{peng2020convolutional,chibane2020implicit}. The above works did not incorporate Euclidean equivariance. As far as we are aware, \cite{deng2021vector} are the first to incorporate Euclidean equivariance in the implicit shape space learning framework. 
    
    Implicit representations are generalized to deformable and articulated shapes by composing the implicit representation with some backward parametric deformation such as Linear Blend Skinning (LBS) \cite{jeruzalski2020nilbs,Saito:CVPR:2021,mihajlovic2021leap}, displacement and/or rotation fields \cite{park2020deformable,pumarola2021d} and  flows \cite{niemeyer2019occupancy,atzmon2021augmenting}. NASA \cite{deng2020nasa} suggest to combine a collection of deformable components represented using individual occupancy networks sampled after reversing the Euclidean transformation of each component. SNARF \cite{chen2021snarf} applies approximated inverse of LBS operator followed by an occupancy query. Both NASA and SNARF work on a single shape and do not learn the latent representation of pose.  \vspace{-10pt}

    \paragraph{Mesh shape space learning.} Mesh shape spaces are often represented as coordinates assigned to a fixed template mesh and GNNs are used to learn their coordinates and latent representations \cite{litany2018deformable,verma2018feastnet,jiang2020disentangled,huang2021arapreg}. \cite{kostrikov2018surface} adapt GNNs to surfaces advocating the dirac operator transferring information from nodes to faces and vice versa.  \cite{litany2018deformable,jiang2020disentangled} use Variational AutoEncoders (VAEs) to improve generalization. The most recent and related work to ours in this domain is \cite{huang2021arapreg} that suggests to incorporate  As-Rigid-As-Possible (ARAP) \cite{sorkine2007rigid,wand2007reconstruction} deformation loss to encourage Euclidean motions of local parts of the shape.

    \section{Method}
    \subsection{Preliminaries: Group action}
    In this work, we consider vector spaces for representing shape and feature spaces. In the following, we define these vector spaces in general terms and specify how the different symmetry groups are acting on them. We use capital letters to represent vector spaces, \eg, $V,W,X,Y,Z$. We use two types of vector spaces: i) $\Real^{a+b\times 3}$, where $a,b\in\Nat_{\geq 0}$ are the invariant and equivariant dimensions,  respectively, and ii) $C^1(\Real^3)$, the space of continuously differentiable scalar volumetric functions. 
    The symmetry considered in this paper is the group of Euclidean motions in $\Real^3$, denoted $E(3)=O(3)\ltimes \Real^3$, where $O(3)$ is the orthogonal matrix group in $\Real^{3\times 3}$. We represent elements in this group as pairs $g=(\mR,\vt)$, where $\mR\in O(3)$ and $\vt\in\Real^3$, where by default vectors are always column vectors.

    The \emph{action} of $G$ on a vector space $V$, denoted $\rho_V$, is defined as follows. First, for $\mV=(\vu,\mU)\in V = \Real^{a+b\times 3}$, consisting of an invariant part $\vu \in \Real^a$, and  equivariant part $\mU \in \Real^{b \times 3}$, we define the action simply by applying the transformation to the equivariant part:
    \begin{equation}\label{e:rho_1}
        \rho_V(g)\mV = (\vu, \mU\mR^T + \one \vt^T)
    \end{equation}
    where $g=(\mR,\vt)\in E(3)$ and $\one\in \Real^{b}$ is the vector of all ones. Second, for $f\in V =  C^1(\Real^3)$ we define the action using change of variables:
    \begin{equation}\label{e:rho_2}
        (\rho_V(g)f)(\vx) = f(\mR^T(\vx-\vt))
    \end{equation}
    for all $\vx\in\Real^3$ and $g=(\mR,\vt)\in G$.


    \subsection{Shape spaces and equivariance}
    We consider an input shape space $X$, a latent space $Z$, and output shape space $Y$, representing shapes in $\Real^3$. All three spaces $X,Z,Y$ are vector spaces as described above, each endowed with an action (using either \eqref{e:rho_1} or \ref{e:rho_2}) of the Euclidean group $G=E(3)$, denoted $\rho_X,\rho_Z,\rho_Y$, respectively. 
    

    Our goal is to learn an encoder $\Phi:X\too Z$, and decoder $\Psi:Z\too Y$ that are \emph{equivariant}. 
    Namely, given an $E(3)$-transformed input, $\rho_X(g)\mX$ we would like its latent code to satisfy
    \begin{equation}
     \Phi(\rho_X(g)\mX)=\rho_Z(g)\Phi(\mX),  
    \end{equation}
    and its reconstruction to satisfy 
    \begin{equation}
     \Psi(\rho_Z(g)\mZ)=\rho_Y(g) \Psi(\mZ).   
    \end{equation}
    Such $X,Z,Y$ are called \emph{steerable} spaces \cite{cohen2016steerable}. The following commutative diagram summarizes the interplay between the encoder, decoder, and the actions of the transformation group:
    \begin{equation*}
        \begin{tikzcd}
    X \arrow{r}{\Phi} \arrow[swap]{d}{\rho_X(g)} & Z \arrow{r}{\Psi} \arrow[swap]{d}{\rho_Z(g)} & \arrow[swap]{d}{\rho_Y(g)} Y \\%
    X \arrow{r}{\Phi}& Z  \arrow{r}{\Psi} & Y
    \end{tikzcd}
    \end{equation*}
    
    \subsection{Frame averaging}\label{ss:frame_averaging}
    We will use  Frame Averaging (FA) \cite{puny2021frame} to build $\Phi,\Psi$. FA allows to build both computationally efficient and maximally expressive equivariant networks.  A \emph{frame} is a map $\gF:V\too 2^G\setminus \emptyset$. That is, for each element $\mV\in V$ it provides a non-empty subset of the group $G=E(3)$, $\gF(\mV)\subset G$. The frame $\gF$ is called \emph{equivariant} if it satisfies 
    \begin{equation}\label{e:frame_equi}
        \gF(\rho_V(g)\mV)=g\gF(\mV)
    \end{equation}
    for all $g\in G$, $\mV\in V$, where for a set $A\subset G$ we define (as usual) $gA=\set{ga\, \vert\, a\in A}$, and the equality in \eqref{e:frame_equi} should be understood in the sense of sets. Then, as shown in \cite{puny2021frame}, an arbitrary map $\phi:V\too W$ can be made equivariant by averaging over an equivariant frame:
    \begin{equation}\label{e:fa}
        \ip{\phi}_\gF(\mV) = \frac{1}{|\gF(\mV)|}\sum_{g\in \gF(\mV)} \rho_W(g) \phi \parr{ \rho_V(g)^{-1} \mV}.
    \end{equation}
    The operator $\ip{\cdot}_\gF$ is called Frame Averaging (FA). An alternative to FA is full group averaging \cite{yarotsky2021universal,chen2021equivariant}, that amounts to replacing the sum over $\gF(\mV)$ in \eqref{e:fa} with an integral over $G$. Full group averaging also provides equivariance and universality. The crucial benefit in FA, however, is that it only requires averaging over a small number of group elements without sacrificing expressive power. In contrast, averaging over the entire group $E(3)$ requires approximating a 6D integral (with an unbounded translation part). Therefore, it can only be approximated and is memory and computationally intensive \cite{chen2021equivariant}. 

    \paragraph{Frame construction.} 
    All the frames we use in this paper are of the form $\gF:V\too 2^G\setminus \emptyset$, for $V=\Real^{d\times 3}$, $G=E(3)$, with the action defined as in \eqref{e:rho_1}.  In some cases we further assume to have some non-negative weight vector $\vw=(w_1,\ldots,w_d)\in\Real^d_{{\scriptscriptstyle \geq 0}}$.
    Given $\mV\in V=\Real^{d\times 3}$ we define $\gF(\mV)\subset E(3)$ using weighted PCA, as follows. First,  
    \begin{equation}\label{e:t}
     \vt = \frac{1}{\one^T\vw}\mV^T \vw   
    \end{equation}
    is the weighted centroid. The covariance matrix is in $\Real^{3\times 3}$ $$\mC = (\mV-\one\vt^T)^T\mathrm{diag}(\vw)(\mV-\one\vt^T),$$
    where $\mathrm{diag}(\vw)\in\Real^{d\times d}$ is a diagonal matrix with $\vw$ along its main diagonal. In the generic case (which we assume in this paper) no eigenvalues of $\mC$ are repeating, \ie, $\lambda_1<\lambda_2<\lambda_3$ (for justification see \eg,  \cite{breiding2018geometry}). Let $\vr_1,\vr_2,\vr_3$ be the corresponding eigenvectors.  
    The frame is defined by $\gF(\mV)=\set{(\mR,\vt)\ \vert \ \mR=\brac{\pm\vr_1,\pm\vr_2,\pm\vr_3}}$, which contains $2^3=8$ elements. Intuitively, $\mV$ is a point cloud in $\Real^3$ and its frame, $\gF(\mV)$, contains all Euclidean motions that take the origin to the weighted centroid of $\mV$ and the axes to the weighted principle directions. The proof of the following proposition is in the supplementary.
    \begin{proposition}\label{prop:frame_equi}
    The frame $\gF$ is equivariant.
    \end{proposition}

    \subsection{Shape space instances}
    \paragraph{Global Euclidean: Mesh $\too$ mesh. }
    In this case, we would like to learn mesh encoder and decoder that are equivariant to global Euclidean motion. We consider the shape spaces $X=Y=\Real^{n\times 3}$ that represent all possible coordinate assignments to vertices of some fixed $n$-vertex template mesh. 
    The latent space is defined as  $Z=\Real^{m+d\times 3}$ consisting of vectors of the form $\mZ=(\vu,\mU)\in Z$, where the $\vu\in\Real^m$ part contains invariant features and the $\mU\in\Real^{d\times 3}$ part contains equivariant features. %
    The group actions $\rho_X,\rho_Z,\rho_Y$ are as defined in \eqref{e:rho_1}. 
    We define our encoder $\Phi$ and decoder $\Psi$ by FA (\eqref{e:fa}), \ie, $    \Phi=\ip{\phi}_{\gF}$, and $\Psi=\ip{\psi}_{\gF}$ where the frames are defined as in Section \ref{ss:frame_averaging} with constant weights $\vw=\one$, $\phi:X\too Z$ and $\psi:Z\too Y$ are standard GNNs adapted to meshes (implementation details are provided in Section \ref{s:implementation}).
    
    \paragraph{Global Euclidean: Point-cloud $\too$ implicit.}
    Here we adopt the setting of \cite{deng2021vector} where $X=\Real^{n\times 3}$ represents all possible $n$-point clouds in $\Real^3$, and $Y=C^1(\Real^3)$ contains implicit representations of a shapes in $\Real^3$. That is, for $f\in Y$ we consider its zero preimage,  \begin{equation}\label{e:level_set}
        f^{-1}(0) = \set{\vx\in\Real^3 \ \vert \ f(\vx)=0}
    \end{equation}
    as our shape rerpesenation in $\Real^3$. If $0$ is a regular value of $f$ then the Implicit Function Theorem implies that $f^{-1}(0)$ is a surface in $\Real^3$. A regular value $r\in \Real$ of $f$ means that at every preimage $\vx\in f^{-1}(r)$, the gradient does not vanish, $\nabla f(\vx)\ne 0$.
    The latent space is again $Z=\Real^{m+d\times 3}$, consisting of vectors of the form $\mZ=(\vu,\mU)\in Z$. The actions $\rho_X,\rho_Z$ are defined as in \eqref{e:rho_1}, while the action $\rho_Y$ is defined as in \eqref{e:rho_2}. 
    The motivation behind the definition of $\rho_Y$ is that $\rho_Y(g) f$ would transform the shape represented by $f$, that is $f^{-1}(0)$, by $g$: 
    \begin{align*}
        (\rho_Y(g)f)^{-1}(0) &= \set{\vx \ \vert \ f(\mR^T(\vx-\vt))=0} \\ 
        &= \set{\mR\vx+\vt \ \vert \ f(\vx)=0} \\
        &= \mR f^{-1}(0) + 
        \vt
    \end{align*}
    
    The encoder is defined as $\Phi=\ip{\phi}_\gF$, where the frames are computed as described in Section \ref{ss:frame_averaging} with constant weights $\vw=\one$, and $\phi:X\too Z$ is a point cloud network (implementation details are provided in Section \ref{s:implementation}). Since the decoder needs to output an element in $Y$, which is a space of functions, we define the decoder by
    \begin{equation}
        \Psi(\mZ) = \hat{\Psi}(\mZ,\cdot),
    \end{equation}  
    where $\hat{\Psi}: Z \times \Real^3 =\Real^{m+3\times(d+1)}\too \Real$.  
    Following \cite{deng2021vector}, to make the decoder $\Psi$ equivariant as a map $Z\too Y$ it is enough to ask that $\hat{\Psi}$ is equivariant under appropriate actions. Namely, the action in \eqref{e:rho_1} applied to $V=\Real^{m+3\times (d+1)}$, and $W=\Real$, where the latter is just the trivial action providing invariance, \ie, $\rho_\Real(g)\equiv 1$.
    \begin{proposition}
    \label{lem:Psi_hat_Psi}
    $\Psi$ is equivariant iff \ $\hat{\Psi}$ is equivariant.
    \end{proposition}
    The decoder is therefore defined as $\hat{\Psi}=\ip{\psi}_\gF$, where $\psi:Z\times \Real^3 \too \Real$ is an MLP (implementation details are provided in Section \ref{s:implementation}), and the frame is defined as in Section \ref{ss:frame_averaging} with constant weights $\vw=\one$.
    
    \begin{figure}[t]
    \centering
         \includegraphics[width=\columnwidth]{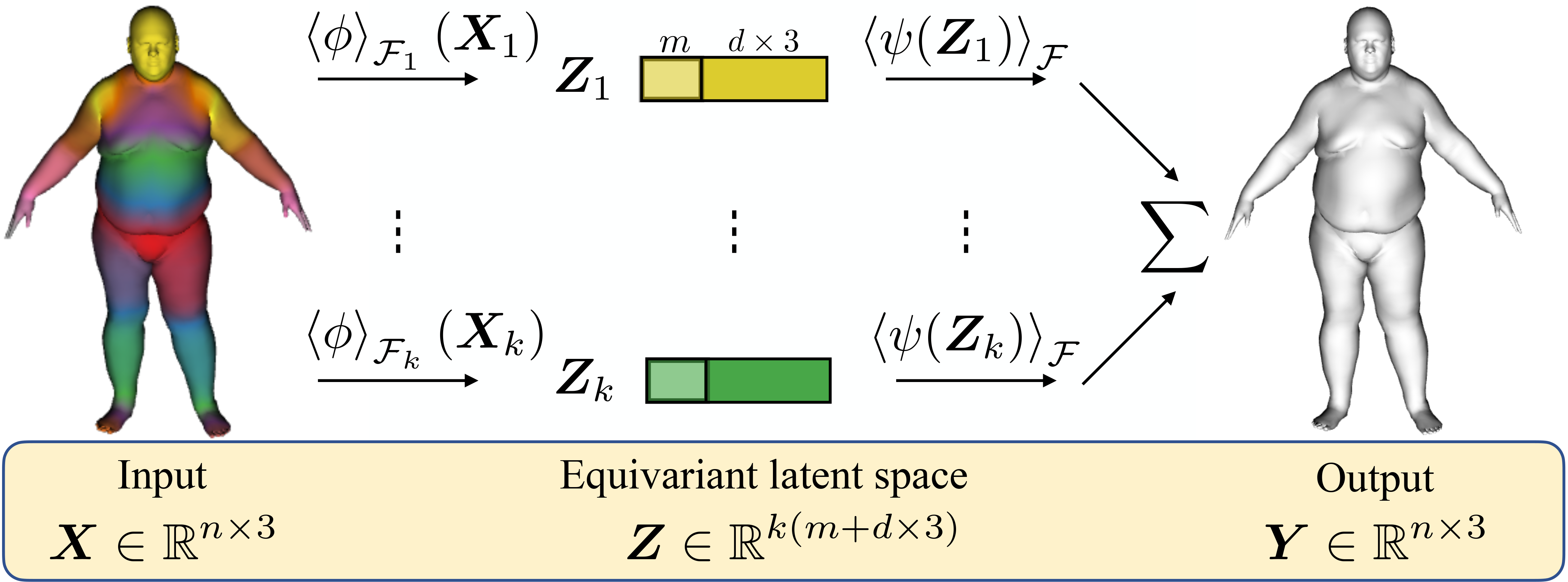}
        \caption{Piecewise Euclidean: Mesh $\too$ mesh. The same $\phi$ backbone is used for the equivariant encoding of each part. Similarly, the same $\psi$ backbone is used for the equivariant decoding of each part's latent code. Lastly, the final prediction is a weighted sum of each part's equivariant output mesh.}
        \label{fig:parts}
        \vspace{-7.5pt}
    \end{figure}
    
    \paragraph{Piecewise Euclidean: Mesh $\too$ mesh.} In this scenario we generalize our framework to be  equivariant to \emph{different} Euclidean motions applied to different object parts (see Figure \ref{fig:parts}). We consider (as before) the shape spaces $X=Y=\Real^{n\times 3}$ that represent all possible coordinate assignments to vertices of some fixed $n$-vertex template mesh.
    The $k$ parts are defined using a partitioning weight matrix (\eg, as used in linear blend skinning) $\mW\in \Real^{n\times k}$, where $\mW_{i,j}\in [0,1]$ indicates the probability that the $i$-th vertex belongs to the $j$-th part, and $\sum_{j=1}^k \mW_{i,j}=1$.
    The latent space has the form $Z=Z_1\times \cdots \times Z_k$, where $Z_j \in \Real^{m+d\times 3}$. Note that $k=1$ represents the case of a global Euclidean motion, as discussed above.
    
    The actions $\rho_X,\rho_Y, \rho_{Z_j}$, $j\in [k]=\set{1,\ldots,k}$, are defined as in \eqref{e:rho_1}. Lastly, we define the encoder and decoder by
    \begin{align}\label{e:local_enc_and_dec}
    \Phi(\mX) &= \parr{ \ip{\phi}_{\gF_j}(\mX_j) \ \vert \ j\in[k]}  \\
    \label{e:piecewise_Psi}
        \Psi(\mZ) &= \sum_{j=1}^k \vw_j  \odot \ip{\psi(\mZ_j)}_{\gF}     
    \end{align}
    where $\phi:X\too Z$, $\psi:Z\too Y$ are Graph Neural Networks (GNNs) as above; $\mX_j\in X$ is the geometry of each part, where all other vertices are mapped to the part's centroid, \ie, $$\mX_j=(\one-\vw_j)\frac{\vw^T\mX}{\vw^T\one}+\vw_j\odot\mX,$$ $\vw_j = \mW_{:,j}$ the $j$-th column of the matrix $\mW$, and each part's frame $\gF_j$ is defined as in Section \ref{ss:frame_averaging} with weights $\vw_j$. The part's latent code is $\mZ_j = \ip{\phi}_{\gF_j}(\mX_j) \in Z_j$. For vector $\va\in\Real^n$ and matrix $\mB\in\Real^{n\times 3}$ we define the multiplication $\va\odot \mB$ by $(\va\odot \mB)_{i,j}=\va_i \mB_{i,j}$. 
    
    If using hard weights, \ie, $\mW\in\set{0,1}^{n\times k}$, this construction guarantees \emph{part-equivariance}. That is, if the $j$-th part of an input shape $\mX\in X$ is transformed by $g_j\in G$, $j\in [k]$, that is,
    \begin{equation*}
        \mX' = \sum_{j=1}^k \vw_j\odot (\rho_X(g_j)\mX)
    \end{equation*}
    then the corresponding latent codes, $\mZ_j$, will be transformed by $\rho_{Z_j}(g_j)$, namely
    \begin{equation*}
        \mZ'_j = \rho_{Z_j}(g_j)\mZ_j
    \end{equation*}
    and the decoded mesh will also transform accordingly,
    \begin{equation*}
        \mY' = \sum_{j=1}^k\vw_j\odot (\rho_Y(g_j)\mY).
    \end{equation*}
    \begin{theorem}\label{thm:part_equi}
    The encoder and decoder in equations \ref{e:local_enc_and_dec} and \ref{e:piecewise_Psi} are part-equivariant.
    \end{theorem}
    In practice we work with a smoothed weighting matrix allowing values in $[0,1]$, \ie, $\mW\in [0,1]^{n\times k}$, losing some of this exact part equivariance for better treatment of transition areas between parts.

    \section{Implementation details}
    \label{s:implementation}
    In this section we provide the main implementation details, further details can be found in the supplementary. 
    \paragraph{Mesh $\too$ mesh.}
    The backbone architectures for $\phi,\psi$ is a $6$ layer GNN, exactly as used in \cite{huang2021arapreg}; The specific dimensions of layers and hidden features for each experiment is detailed in the supplementary appendix. We denote the learnable parameters of both networks by $\theta$. The training loss is the standard autoencoder reconstruction loss of the form %
    \begin{equation}\label{e:loss_rec}
        \gL_{\text{rec}}(\theta) = \frac{1}{N}\sum_{i=1}^N \norm{ \Psi(\Phi(\mX^{(i)})) - \mX^{(i)} }_F
    \end{equation}
    where $\norm{\cdot}_F$ is the Frobenious norm and $\set{\mX^{(i)}}_{i=1}^N\subset\Real^{n\times 3}$ is a batch drawn from the shape space's training set. 
    
    \paragraph{Point cloud $\too$ implicit.} 
    The backbone encoder architecture $\phi$ is exactly as in \cite{mescheder2019occupancy} constructed of PointNet \cite{qi2017pointnet} with $4$ layers. The decoder is an MLP as in \cite{atzmon2020sald} with $8$ layers with $512$ features each. We trained a VAE where the latent space is $Z=\Real^{d + m +d\times 3}$ containing codes of the form $(\vmu,\veta)$, where $\vmu\in\Real^{m+d\times 3}$ is the latent mean, and $\veta\in\Real^{d}$ is the invariant latent log-standard-deviation. For training the VAE we use a combination of two losses
    \begin{equation}
        \gL(\theta) = \gL_{\text{sald}}(\theta) + 0.001 \gL_{\text{vae}}(\theta),
    \end{equation}
    where $\gL_{\text{sald}}$ is the SALD loss \cite{atzmon2020sald},
    \begin{equation}
        \gL_{\text{sald}}(\theta) = \frac{1}{N}\sum_{i=1}^N \int_\Omega \tau(\Psi(\gN(\vmu^{(i)},\veta^{(i)})), h)(\vx)d\vx
    \end{equation}
    where $(\vmu^{(i)},\veta^{(i)})=\Phi(\mX^{(i)})$, $\gN(\va,\vb)$ is an axis aligned Gaussian i.i.d.~sample with mean $\va$ and standard deviation $\exp(\diag(\vb))$. $h(\cdot)$ is the unsigned distance function to $\mX^{(i)}$, and $\tau(f,g)(\vx) = \abs{\abs{f(\vx)} - g(\vx)} + \min\set{ \norm{\nabla f(\vx)-\nabla g(\vx)}_2 , \norm{\nabla f(\vx)+\nabla g(\vx)}_2}$. The domain of the above integral, $\Omega\subset \Real^3$, is set according to the scene's bounding box. In practice, the integral is approximated using Monte-Carlo sampling. Note that this reconstruction loss is unsupervised (namely, using only the input raw point cloud). The VAE loss is defined also as in \cite{atzmon2020sald} by  
    \begin{equation}
        \gL_{\text{vae}}(\theta) = \sum_{i=1}^N  \norm{\vmu^{(i)}}_1 + \norm{\veta^{(i)} + \one}_1,
    \end{equation} 
    where $\norm{\cdot}_1$ denotes the $1$-norm.

    \section {Experiments}\label{sec:exp}
    
    
    We have tested our FA shape space learning framework under two kinds of symmetries $G$: global Euclidean transformations and piecewise Euclidean transformations. 
    
    
    \subsection{Global Euclidean}
    In this case, we tested our method both in the mesh $\too$ mesh and the point-cloud $\too$ implicit settings. 
    
    \begin{table}[t]
        \centering
        \scriptsize
        \setlength\tabcolsep{8pt} 
        \begin{tabular}{c}
            \begin{adjustbox}{max width=\textwidth}
                \aboverulesep=0ex
                \belowrulesep=0ex
                \renewcommand{\arraystretch}{1.1}
                \begin{tabular}{l|c|c|c}
                Method &  $I$ & $z$ & $SO(3)$  \\ 
                        \hline
                AE  & 5.16 & 9.96 & 15.41   \\
                AE-Aug & 5.22 & 5.86 &  5.12 \\
                Ours  & \textbf{4.39} & \textbf{4.35} & \textbf{4.66}   
                \end{tabular}
                \end{adjustbox}
    \end{tabular}
    \caption{ Global Euclidean mesh$\too$mesh shape space experiment; MSE error (lower is better) in three test versions of the DFAUST \cite{dfaust:CVPR:2017} dataset, see text for details.  }
    \label{tab:mesh_global_rigid}
    \end{table}
    
    \vspace{-10pt}
    \paragraph{Mesh $\rightarrow$ mesh.}
    In this experiment we consider the DFaust dataset \cite{dfaust:CVPR:2017} of human meshes parameterized with SMPL \cite{SMPL:2015}. The dataset consists of 41,461 human shapes where a random split is used to generate a training set of 37,197 models, and a test set of 4,264 models. We used the same generated data and split as in \cite{huang2021arapreg}. We generated two additional test sets of randomly oriented models: randomly rotated models about the up axis (uniformly), denoted by $z$, and randomly rotated models (uniformly), denoted by $SO3$. We denote the original, aligned test set by $I$. We compare our global Euclidean mesh$\too$mesh autoencoder versus the following baselines: Vanilla Graph autoencoder, denoted by AE; and the same AE trained with random rotations augmentations, denoted by AE-Aug. Note that the architecture used for AE and AE-Aug is the same backbone architecture used for our FA architecture. Table \ref{tab:mesh_global_rigid} reports the average per-vertex euclidean distance (MSE) on the various test sets: $I$, $z$ and $SO3$. Note that FA compares favorably to the baselines in all tests.

    \begin{table}[h]
        \centering
        \scriptsize
        \setlength\tabcolsep{4.0pt} 
        \begin{tabular}{c}
            \begin{adjustbox}{max width=\textwidth}
                \aboverulesep=0ex
                \belowrulesep=0ex
                \renewcommand{\arraystretch}{1.0}
                \begin{tabular}[t]{l|cc|cc|cc|cc}
                \multicolumn{1}{l}{} &  \multicolumn{2}{c}{teddy bear} & \multicolumn{2}{c}{bottle} & \multicolumn{2}{c}{suitcase} & \multicolumn{2}{c}{banana} \\ 
                Method
                & $\dist_{\text{C}}^{\rightarrow}$ & $\dist_{\text{C}}$  & $\dist_{\text{C}}^{\rightarrow}$ & $\dist_{\text{C}}$ & 
                $\dist_{\text{C}}^{\rightarrow}$ & $\dist_{\text{C}}$ & 
                $\dist_{\text{C}}^{\rightarrow}$ & $\dist_{\text{C}}$ \\
                
                \hline
                        
                VAE  & 5.11 & 2.611 & 0.419 & 0.225 & 0.619 & 0.341 & 0.309 & 0.177    \\
                VN \cite{deng2021vector}  & 0.047 & \textbf{0.421} & 0.638 & 0.334 & 0.348 & 0.218 & 0.157 & 0.087   \\
                Ours & \textbf{0.046} & 0.451 &  \textbf{0.226} & \textbf{0.129} & \textbf{0.079} & \textbf{0.086} & \textbf{0.118} & \textbf{0.074}
                
                \end{tabular}
                \end{adjustbox}
    \end{tabular}
    \caption{Global Euclidean point cloud $\too$ implicit shape space experiment; CommonObject3D \cite{reizenstein2021common} dataset.}
    \label{tab:common3D}
    \end{table}
    \begin{figure}[t]
    \centering
         \includegraphics[width=\columnwidth]{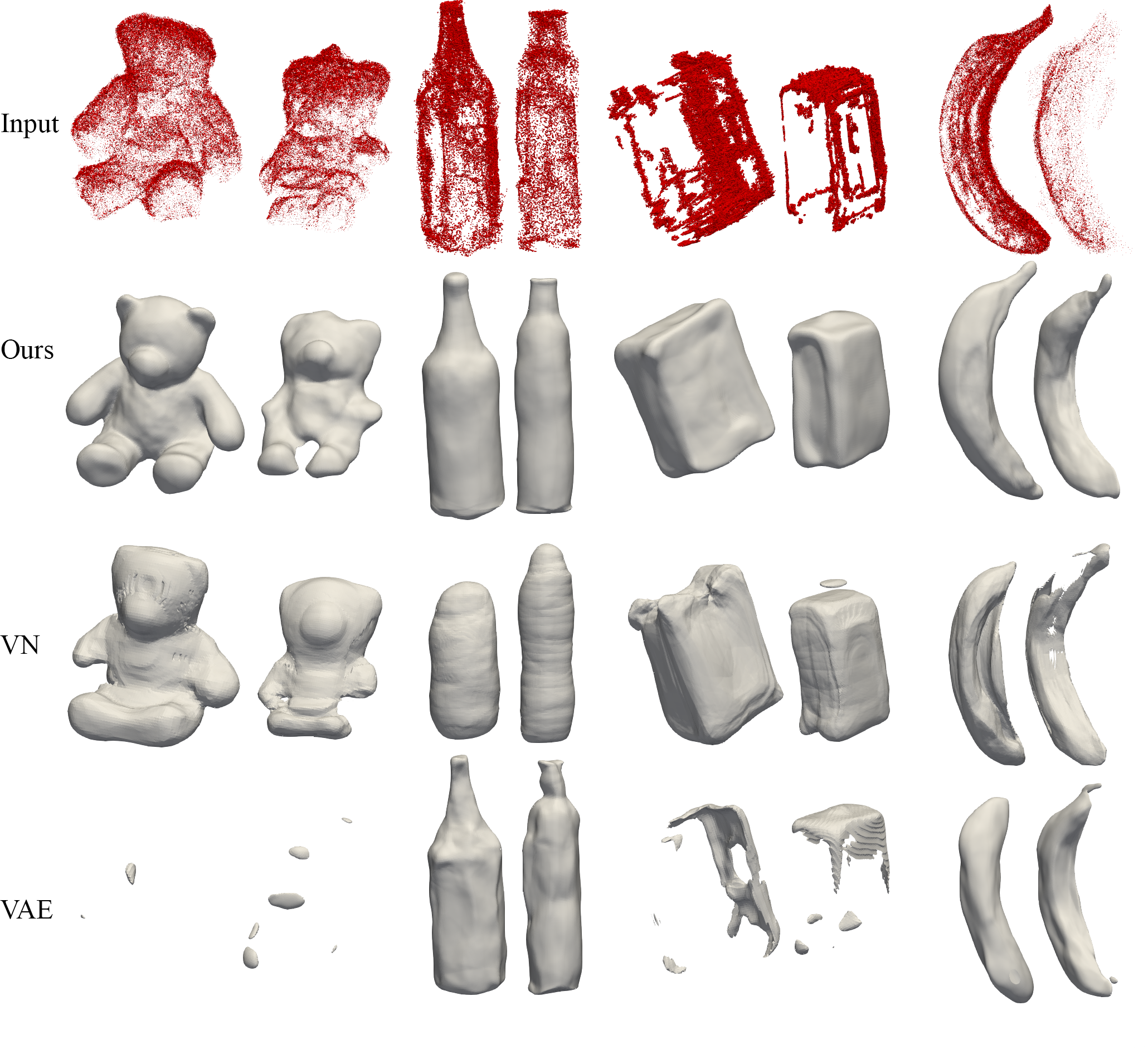}
        \caption{Global Euclidean point cloud $\too$ implicit, qualitative test results; CommonObject3D \cite{reizenstein2021common} dataset.}
        \label{fig:common3D}
        \vspace{-15pt}
    \end{figure}
    \vspace{-10pt}
    \paragraph{Point cloud $\too$ implicit.}
    In this experiment we consider the CommonObject3D dataset \cite{reizenstein2021common} that contains 19k objects from 50 different classes. We have used only the objects' point clouds extracted from videos using COLMAP \cite{schoenberger2016sfm}. The point clouds are very noisy and partial (see \eg, Figure \ref{fig:common3D}, where input point clouds are shown in red), providing a "real-life" challenging dataset. Note that we have not used any other supervision for the implicit learning. We have used 4 object catagories: teddy bear ($747$ point clouds), bottle (296 point clouds), suitcase (480 point clouds), and banana (197 point clouds). We have divided each category to train and test sets randomly based on a 70\%-30\% split. We compare to the following baselines: Variational Autoencoder, denoted by VAE; Vector Neurons \cite{deng2021vector} version of this VAE, denoted VN. We used the official VN implementation. For our method we used an FA version of the same baseline VAE architecture. Table \ref{tab:common3D} reports two error metrics on the test set: $\dist_{\text{C}}^{\rightarrow}$ that denotes the one sided Chamfer distance from the input point cloud to the generated shape  and $\dist_{\text{C}}$ that denotes the symmetric Chamfer distance (see supplementary for exact definitions). Note that our method improves the symmetric Chamfer metric in almost all cases. Figure \ref{fig:common3D} shows some typical reconstructed implicits (after contouring) in each category, along with the input test point cloud (in red). Qualitatively our framework provides a more faithful reconstruction, even in this challenging noisy scenario without supervision. Note that for the "teddy bear" class we provide a visually improved reconstruction that is not apparent in the qualitative score due to high noise and outliers in the input point clouds.



    \begin{figure*}
        \centering
        \includegraphics[width=1.0\textwidth]{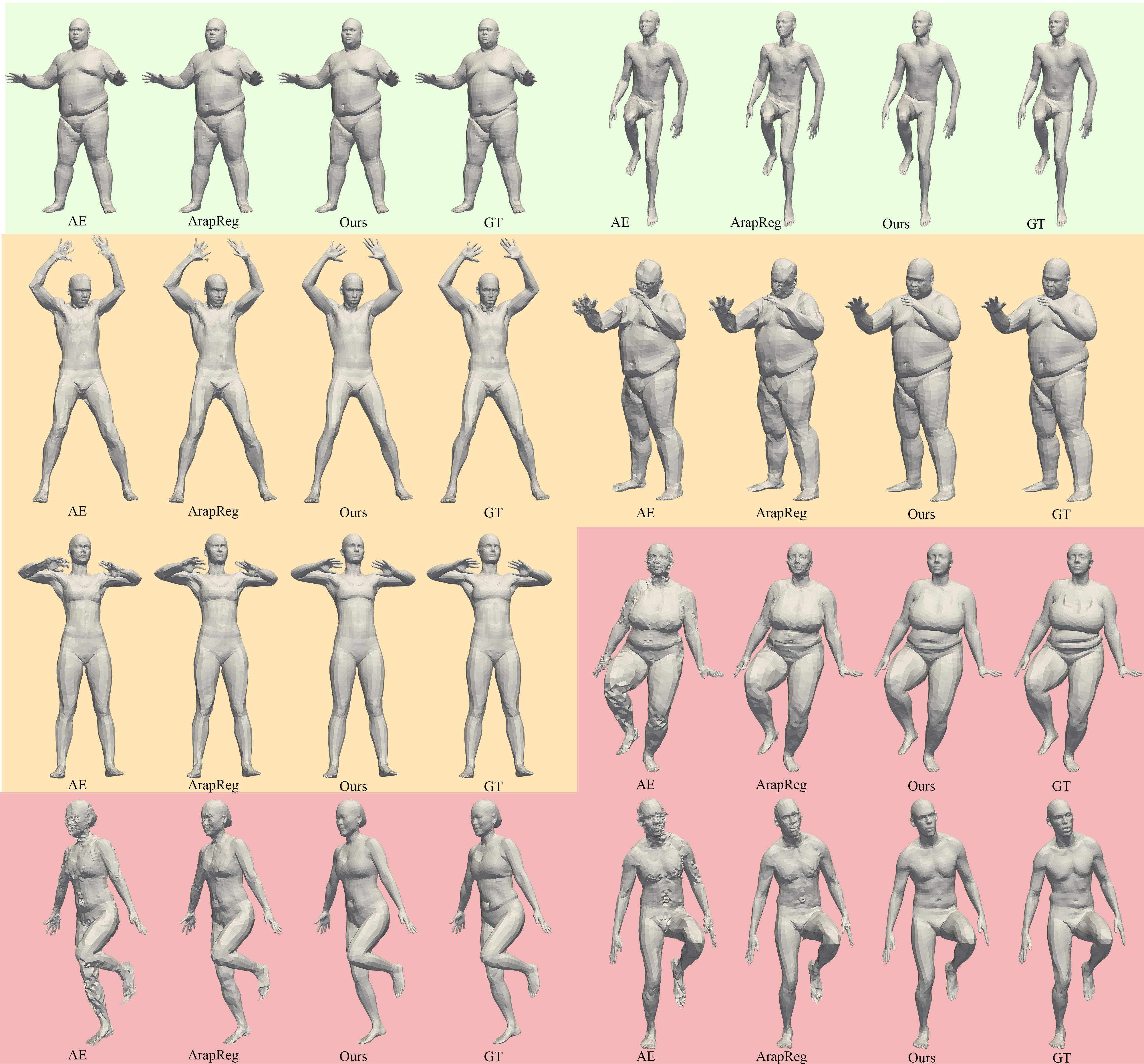}
        \caption{Piecewise Euclidean mesh $\too$ mesh, qualitative results; DFaust \cite{dfaust:CVPR:2017} dataset. Colors mark different splits: green is the random (easy) split; orange is the unseen random pose split; and red is the unseen pose split, see text for details. Our method demonstrates consistently high-quality results across splits of different difficulty levels.\vspace{10pt} }
        \label{fig:piecewise}
        \vspace{-12pt}
    \end{figure*}

    \subsection{Piecewise Euclidean}
    \vspace{-2pt}
    \paragraph{Mesh $\too$ mesh.}
    In this experiment we consider three different datasets: DFaust \cite{dfaust:CVPR:2017}, SMAL \cite{zuffi20173d} and MANO \cite{romero2020group}. For the DFaust dataset we used train-test splits in an increasing level of difficulty: random split of train-test taken from \cite{huang2021arapreg}, as described above; unseen random pose - removing a random (different) pose sequence from each human and using it as test; and unseen pose - removing the \emph{same} pose sequence from all humans and using it as test. The SMAL dataset contains a four-legged animal in different poses. We use the data generated in \cite{huang2021arapreg} of 400 shapes, randomly split into a $300$ train, $100$ test sets. The MANO dataset contains 3D models of realistic human hands in different poses. Using the MANO SMPL model we generated 150 shapes, randomly split into a $100$ train and $50$ test sets. We compared our method to the following baselines: Vanilla autoencoder, denoted by AE and ARAPReg \cite{huang2021arapreg} that reported state of the art results on this dataset. Note that both the AE and our method use the same backbone architecture. ARAPReg, report autodecoder to be superior in their experiments and therefore we compared to that version. Note that all compared methods have the same (backbone) decoder architecture.  Figure \ref{fig:piecewise} shows typical reconstruction results on the tests sets: green marks the random (easy) split; orange marks the random unseen pose split; and red marks the global unseen pose split. Note our method is able to produce very high-fidelity approximations of the ground truth, visually improving artifacts, noise and inaccuracies in the baselines (zoom in to see details). Lastly, we note that we use also partition skinning weight matrix (defined on the single rest-pose model) as an extra supervision not used by ARAPReg. 
    
    
    
    \begin{table}[h]
        \centering
        \scriptsize
        \setlength\tabcolsep{3.8pt} 
        \begin{tabular}{c}
            \begin{adjustbox}{max width=\textwidth}
                \aboverulesep=0ex
                \belowrulesep=0ex
                \renewcommand{\arraystretch}{1.1}
                \begin{tabular}{l|c|c|c|c|c}
                Method &  random & unseen random pose & unseen pose & SMAL & MANO
                     \\ 
                        \hline
                AE  & 5.45 & 7.99 & 6.27 & 9.11 & 1.34 \\
                ARAPReg & 4.52 & 7.77 & 3.38 & 6.68 & 1.15 \\
                Ours  & \textbf{1.68} & \textbf{1.89} & \textbf{1.90} & \textbf{2.44}  & \textbf{0.86} 
                \end{tabular}
                \end{adjustbox}
    \end{tabular}
    \caption{ Piecewise Euclidean mesh $\too$ mesh experiment; MSE error (lower is better); DFaust \cite{dfaust:CVPR:2017}, SMAL \cite{zuffi20173d} and MANO \cite{romero2020group} datasets. \vspace{-10pt}}
    \label{tab:normal}
    \end{table}

    \begin{table*}[t]
        \centering
        \scriptsize
        \setlength\tabcolsep{6pt} 
        \begin{tabular}{c}
            \begin{adjustbox}{max width=\textwidth}
                \aboverulesep=0ex
                \belowrulesep=0ex
                \renewcommand{\arraystretch}{1.0}
                \begin{tabular}[t]{c|ccc|ccc|ccc|ccc}
                \multicolumn{1}{c}{} & 
                \multicolumn{6}{c}{Within Distribution} & 
                \multicolumn{6}{|c}{Out of Distribution} \\
                \cmidrule{2-13}
                \multicolumn{1}{c}{} &
                \multicolumn{3}{c}{IoU bbox} &
                \multicolumn{3}{c}{IoU surface} &
                \multicolumn{3}{c}{IoU bbox} &
                \multicolumn{3}{c}{IoU surface}  \\
                & NASA & SNARF & Ours  & NASA & SNARF & Ours & NASA & SNARF & Ours & NASA & SNARF & Ours \\
                    \midrule
                    50002 &
                     96.56\% & 97.50\% & \textbf{98.67\%} & 
                     84.02\% & 89.57\% & \textbf{93.28\%} &
                     87.71\% & 94.51\% &  \textbf{96.76\%}&
                     60.25\% & 79.75\% & \textbf{85.06\%}  \\
                     
                     50004 &
                     96.31\% & 97.84\% & \textbf{98.64\%} & 
                     85.45\% & 91.16\% & \textbf{94.57\%} &  
                     86.01\% & 95.61\% & \textbf{96.19\%} &
                     62.53\% & 83.34\% & \textbf{85.84\%}  \\
                    
                    50007 &
                     96.72\% & 97.96\% & \textbf{98.62\%} &  
                     86.28\% & 91.02\% & \textbf{94.11\%}  &
                     80.22\% & 93.99\% & \textbf{95.31\%} &
                     51.82\% & 77.08\% & \textbf{81.91\%} \\
                     50009 &
                     94.96\% & 96.68\% & \textbf{97.75\%} & 
                     84.52\% & 88.19\% & \textbf{92.84\%}  &
                     78.15\% & 91.22\% & \textbf{94.75\%} &
                     55.86\% & 75.84\% & \textbf{84.60\%} \\
                     50020 &
                     95.75\% & 96.27\% & \textbf{97.61\%} &
                     87.57\% & 88.81\% &  \textbf{92.60\%} &  
                     83.06\% & 93.57\% & \textbf{95.17\%} &
                     62.01\% & 81.37\% & \textbf{85.66\%} \\
                     50021 &
                     95.92\% & 96.86\% & \textbf{98.55\%} &
                     87.01\% & 90.16\% & \textbf{95.38\%}  &
                     81.80\% & 93.76\% & \textbf{96.35\%} &
                     65.49\% & 81.49\% & \textbf{88.86\%} \\
                     50022 &
                     97.94\% & 97.96\% & \textbf{98.39\%} &
                     91.91\% & 92.06\% &  \textbf{93.68\%}  &
                     87.54\% & 94.67\% & \textbf{96.12\%} &
                     70.23\% & 83.37\% & \textbf{85.80\%} \\
                     50025 &
                     95.50\% & 97.54\% & \textbf{98.48\%} & 
                     86.19\% & 91.25\% & \textbf{94.74\%}  &
                     83.14\% & 94.48\% & \textbf{95.99\%} &
                     60.88\% & 82.48\% & \textbf{86.58\%} \\
                     50026 &
                     96.65\% & 97.64\% & \textbf{98.61\%} &
                     87.72\% & 91.09\% & \textbf{94.64\%} &
                     84.58\% & 94.13\% & \textbf{96.45\%} &
                     59.78\% & 80.01\% & \textbf{87.10\%} \\
                     50027 &
                     95.53\% & 96.80\% & \textbf{97.95\%}  &
                     86.13\% & 89.47\% & \textbf{93.46\%} & 
                     83.97\% & 93.76\% & \textbf{95.61\%} &
                     61.82\% & 81.81\% & \textbf{86.60\%} 
            \end{tabular} 
            \end{adjustbox}
          
        \end{tabular}
        \vspace{-3pt}
        \caption{Piecewise Euclidean mesh $\too$ mesh, comparison to implicit articulation methods. DFaust \cite{dfaust:CVPR:2017} and PosePrior \cite{akhter2015pose} datasets. \vspace{10pt}}
        \label{tab:all_humans_splits}
        \vspace{-18pt}
    \end{table*}
    
    \vspace{-8pt}
    \paragraph{Interpolation in shape space.}
    In this experiment we show qualitative results for interpolating two latent codes $\mZ^{(j)}=(\vq^{(j)},\mQ^{(j)})\in Z$, $j=0,1$, computed with our encoder for two input shapes $\mX^{(j)}$, $j=0,1$. We use the encoder and decoder learned in the "unseen pose" split described above. Since $Z$ is an equivariant feature space, if $\mX^{(1)}$ is an Euclidean transformed version of $\mX^{(0)}$, \ie, $\mX^{(1)} = \rho_X(g)\mX^{(0)}$, then equivariance would mean that $\mZ^{(1)}=\rho_Z(g)\mZ^{(0)}$. Therefore interpolation in this case should be done by finding the optimal rotation and translation between the equivariant parts of $\mZ^{(0)}$ and $\mZ^{(1)}$ and continuously rotating and translating $\mZ^{(0)}$ into $\mZ^{(1)}$. This can be done using the closed form solution to the rotational Procrustes problem (see \eg,  \cite{zhu20073d,schonemann1966generalized}). For two general codes $\mZ_j$ we use this procedure while adding linearly the residual difference after cancelling the optimal rotation and translations between the codes. In the supplementary we provide the full derivation of this interpolation denoted $\mZ^{t}$, $t\in[0,1]$. Figure \ref{fig:interp} shows the result of decoding the interpolated latent codes $\mZ^t$, $t\in[0,1]$ with the learned decoder. Note that both shape and pose gracefully and naturally change along the interpolation path.

    \begin{figure*}
        \centering
        \includegraphics[width=1.0\textwidth]{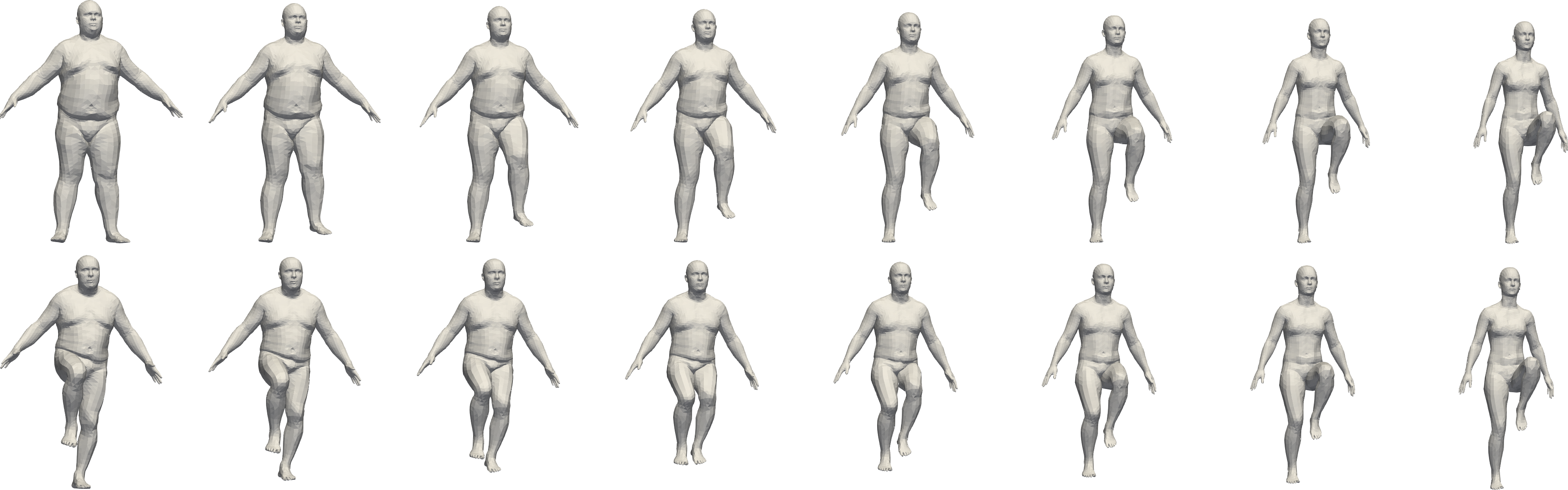}
        \caption{Interpolation in equivariant latent space between two test examples from the "unseen pose" split (leftmost and rightmost columns). \vspace{0pt}}
        \label{fig:interp}
        \vspace{-15pt}
    \end{figure*}
    
    \paragraph{Comparison with implicit pose-conditioned methods.}
    Lastly, we have trained our piecewise Euclidean mesh $\too$ mesh framework on the DFaust subset of AMASS \cite{mahmood2019amass}. Following the protocol defined in \cite{chen2021snarf}, we trained our model on each of the $10$ subjects and tested it on the "within distribution" SMPL \cite{SMPL:2015} generated data from SNARF \cite{chen2021snarf}, as well as on their "out of distribution" test from PosePrior \cite{akhter2015pose} dataset. We use both SNARF \cite{chen2021snarf} and NASA \cite{deng2020nasa} as baselines. Table \ref{tab:all_humans_splits} reports quantitative error metrics: The Intersection over Union (IoU) using sampling within the bounding box (bbox) and near the surface, see supplementary for more details. Figure \ref{fig:geiger} shows comparison with SNARF of test reconstructions from the "out of distribution" set. We note that our setting is somewhat easier than that faced by SNARF and NASA (we perform mesh $\too$ mesh learning using a fixed mesh connectivity and skinning weights; the skinning weights is used by NASA and learned by SNARF).

    Nevertheless, we do not assume or impose anything on the latent space besides Euclidean equivariance, not use an input pose explicitly, and train only with a simple reconstruction loss (see \eqref{e:loss_rec}). Under this disclaimer we note we improve the reconstruction error both qualitatively and quantitatively in comparison to the baselines.

    \begin{figure}[t]
    \centering
         \includegraphics[width=\columnwidth]{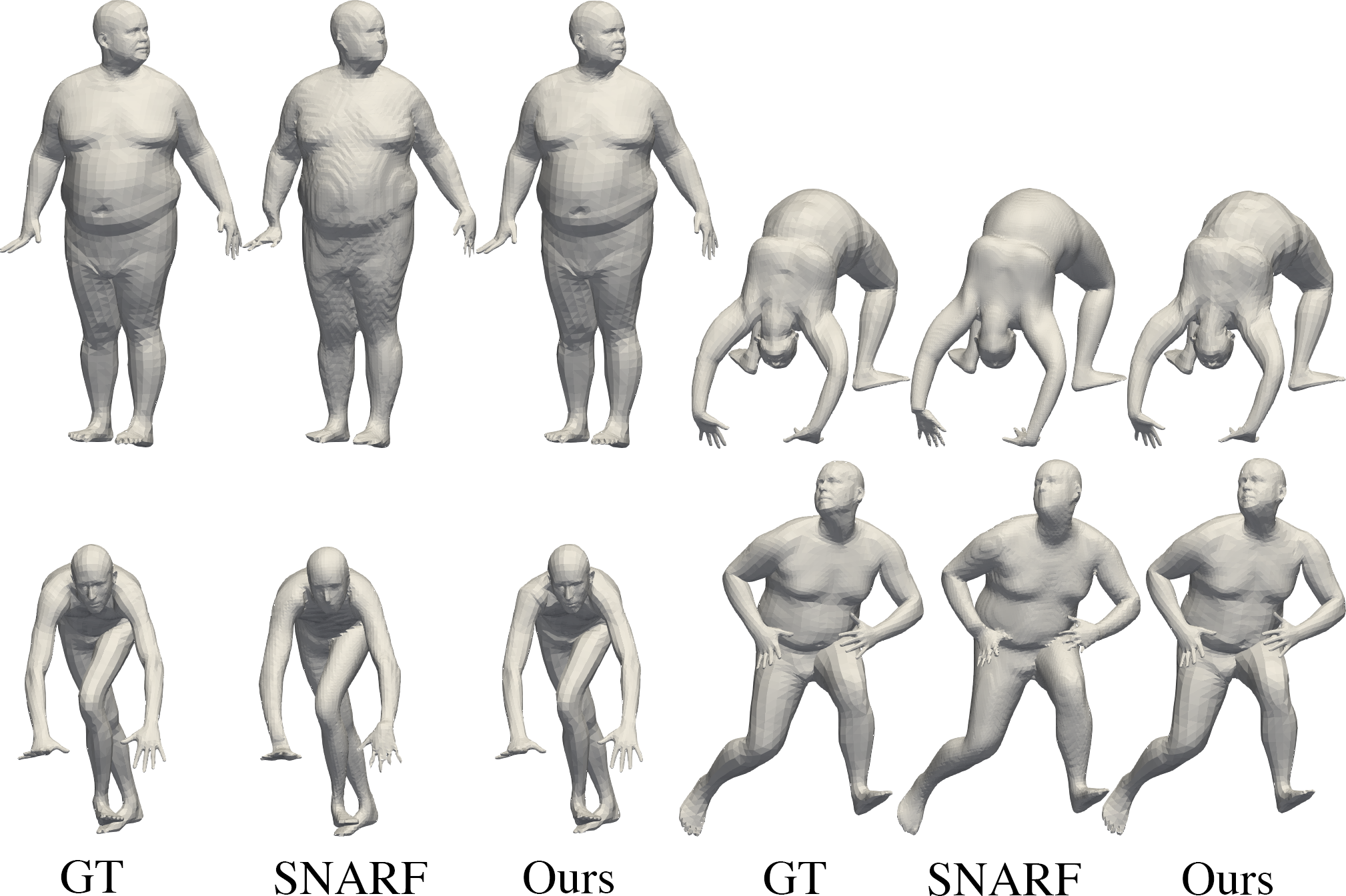}
        \caption{Comparison with SNARF on the "out of distribution" test set.}
        \label{fig:geiger}
        \vspace{-25pt}
    \end{figure}
    \vspace{-6pt}
    \section {Limitations and future work}
    We have introduced a generic methodology for incorporating symmetries by construction into encoder and/or decoders in the context of shape space learning. Using Frame Averaging we showed how to construct expressive yet efficient equivariant autoencoders. We instantiated our framework for the case of global and piecewise Euclidean motions, as well as to mesh $\too$ mesh, and point cloud $\too$ implicit scenarios. We have achieved state of the art quantitative and qualitative results in all experiments. 
    
     Our method has several limitations: First in the mesh $\too$ mesh case we use fixed connectivity and skinning weights. Generalizing the piecewise Euclidean case to implicit representations, dealing with large scale scenes with multiple objects, or learning the skinning weights would be interesting future works. Trying to use linear blend skinning to define group action of $E(3)^k$ could be also interesting. Finally, using this framework to explore other symmetry types, combinations of geometry representations (including images, point clouds, implicits, and meshes), and different architectures could lead to exciting new methods to learn and use shape space in computer vision.  

    {\small
    \bibliographystyle{ieee_fullname}
    \bibliography{egbib}
    }

    \section{Additional Details}
    Here we provide additional details for each experiment in section \ref{sec:exp}. 
    
    \subsection{Global Euclidean Mesh $\too$ mesh.}
    \paragraph{Architecture.}\label{sss:arch_global_mesh} We start by describing in full details the  backbone architectures for both the encoder $\phi$ and decoder $\psi$, for all the methods in table \ref{tab:mesh_global_rigid}. Then, we describe the relevant hyperparmaeters for our Frame Averaging version. Note that the backbone architectures are similar to the ones in ~\cite{huang2021arapreg}, and are described here again for completeness. The network consists of the following layers:
    \begin{align*}
    \mathrm{Conv}(n,d_{\text{in}},d_{\text{out}}) &:\Real ^ {n \times d_\text{in}} \too \Real ^{ n \times d_{\text{out}}} \\
    \mathrm{Pool}(n,n') &: \Real^{n \times d} \too  \Real ^ {n' \times d} \\
    \mathrm{Linear}(d_{\text{in}},d_{\text{out}}) &:\Real ^{d_{\text{in}}} \too \Real ^{d_{\text{out}}}.
    \end{align*}
    The $\mathrm{Conv}(n,d_{\text{in}},d_{\text{out}})$ is the Chebyshev spectral graph convolutional operator \cite{defferrard2016convolutional} as implemented in \cite{ranjan2018generating,huang2021arapreg}, where $n$ denotes the number of vertices, and $d_{\text{in}}$, $d_{\text{out}}$ are the layer's input and output features dimensions respectively. All the $\mathrm{Conv}(n,d_{\text{in}},d_{\text{out}})$ layers have $6$ Chebyshev polynomials. The $\mathrm{Pool}(n,n')$ layer is based on the mesh down/up sampling schemes defined in \cite{ranjan2018generating}, where we used the same mesh samples as provided in \cite{huang2021arapreg} official implementation. The $n$, $n'$ are the number of vertices in the input and the output, respectively. The linear layer is defined by $\mX' = \mW \mX + \vb $, where $\mW \in \Real ^{d_{\text{out}} \times d_{\text{in}}}$ and $\vb \in \Real ^{d_{\text{out}}}$ are learnable parameters.  
    
    Moreover, the backbone architecture consists of the following blocks:
    \begin{flalign*}
     &\mathrm{EnBlock}(n,n',d_{\text{in}},d_{\text{out}}):\mX \in \Real ^ {n \times d_\text{in}} \mapsto \mX' \in \Real ^{ n \times d_{\text{out}}} \\
    \mX' &= \mathrm{Pool}(n,n',d_{\text{out}})\left(\sigma \left(\mathrm{Conv}(n,d_{\text{in}},d_{\text{out}}) \right)\right) \\
    &\mathrm{DeBlock}(n, n',d_{\text{in}},d_{\text{out}}):\mX \in \Real ^ {n \times d_\text{in}} \mapsto \mX' \in \Real ^{ n \times d_{\text{out}}} \\
    \mX' &= \sigma \left(\mathrm{Conv}(n',d_{\text{in}},d_{\text{out}})\left( \mathrm{Pool}(n,n',d_{\text{in}})(\mX)\right) \right)
    \end{flalign*}
    where $\sigma$ is the $\mathrm{ELU}$ activation function. 
    Then, the encoder $\phi: \Real^{6890 \times 3} \too \Real^{72}$ consists of the following blocks:
    \begin{align*}
    &\mathrm{EnBlock}(6890,3445,3,32) \rightarrow \\
    &\mathrm{EnBlock}(3445,1723,32,32)  \rightarrow \\
    &\mathrm{EnBlock}(1723,862,32,32) \rightarrow \\
    &\mathrm{EnBlock}(862,431,32,64) \rightarrow \mathrm{Linear}(27584,72)  \rightarrow.
    \end{align*}
    
    Similarly, the decoder $\psi : \Real ^ {72} \too \Real ^{n \times 3}$ is consisted of the following blocks:
    \begin{align*}
    &\mathrm{Linear}(72,27584)  \rightarrow \mathrm{DeBlock}(431,862,64,64) \rightarrow \\ 
    &\mathrm{DeBlock}(862,1723,64,32) \rightarrow \\
    &\mathrm{DeBlock}(1723,3445,32,32) \rightarrow \\
    &\mathrm{DeBlock}(3445,6890,32,32) \rightarrow \\
    &\mathrm{Conv}(6890,6890,32,3)  \rightarrow.
    \end{align*}
    
    For the Frame Averaging versions of these backbone architectures, $\Phi=\ip{\phi}_{\gF}$, and $\Psi=\ip{\psi}_{\gF}$, we set $m=0$ and $d=24$. For the frames construction of $\Phi$ and $\Psi$, $\gF:V\too 2^{E(3)}\setminus \emptyset$, we use the construction defined in section \ref{ss:frame_averaging}, with $\vw=\one$. Note that for $\Phi$ the frames are calculated with respect to the input vertices $\mV = \mX \in \Real ^{6890 \times 3}$, and for $\Psi$ the frames are calculated with respect to $\mV = \mZ \in \Real^{24 \times 3}$.
    
    \paragraph{Implementation details.}
    For all methods in table \ref{tab:mesh_global_rigid}, training was done using the \textsc{Adam} \cite{kingma2014adam} optimizer, with batch size $64$. The number of epochs was set to $150$, with $0.0001$ fixed learning rate. Training was done on $2$ Nvidia V-100 GPUs.
    
    \paragraph{Evaluation.}
    The quantitative measure reported in table \ref{tab:mesh_global_rigid}, the average per-vertex distance error (MSE), is based on the official implementation of \cite{huang2021arapreg}. Given a collection of $N$ test examples $\left\{\mX_i\right\}_{i=1}^N$, and corresponding predictions $\left\{\mY_i\right\}_{i=1}^N$, the metric is defined by 
    \begin{equation}\label{eq:mse}
    \text{MSE} = \frac{1}{NM} \sum_{i=1}^N \sum_{j=1}^M \norm{\mX_{ij} - \mY_{ij}},
    \end{equation}
    where $M$ is the number of vertices.
    
    \subsection{Global Euclidean: Point cloud $\too$ implicit}
    \paragraph{Architecture.}
    In this experiment, the backbone encoder architecture is the same for the baseline VAE and our Frame Averaging version. The network is based on the design and implementation of OccupencyNetworks \cite{mescheder2019occupancy}. The encoder network consists of layers of the form 
    \begin{align*}
    \mathrm{FC}(d_{\text{in}},d_{\text{out}})&:\mX \in \Real^{n \times d_{\text{in}}} \mapsto \nu\parr{\mX \mW + \one \vb^T } \\
    \mathrm{MaxPool} &: \mX \in \Real^{n \times d} \mapsto  [\max{\mX \ve_i}] \\
    \mathrm{Linear}(d_{\text{in}},d_{\text{out}})&:\mV \in \Real^{d_{\text{in}}} \mapsto \mW' \mV +  \vb'
    \end{align*}
    where $\mW\in \Real^{d_\text{in}\times d_\text{out}}$, $\mW'\in \Real^{d_\text{out}\times d_\text{in}}$, $\vb\in\Real^{d_\text{out}}$, $\vb'\in\Real^{d_\text{out}}$ are the learnable parameters, $\one\in\Real^n$ is the vector of all ones, $[\cdot]$ is the concatenation operator, $\ve_i$ is the standard basis in $\Real^{d}$, and $\nu$ is the $\mathrm{ReLU}$ activation. Then, we define the following block:
    \begin{align*}
    &\mathrm{EnBlock}(d_{\text{in}},d_{\text{out}}):\mX \mapsto \\ &[\mathrm{FC}(d_{\text{in}},d_{\text{out}}/2)(\mX),\one \mathrm{MaxPool}(\mathrm{FC}(d_{\text{in}},d_{\text{out}}/2)(\mX)],
    \end{align*}
    and the encoder $\phi$ consists of the following blocks:
    \begin{align*}
    & \mathrm{FC}(3,510) \rightarrow \mathrm{EnBlock}(510,510)  \rightarrow \\
    &\mathrm{EnBlock}(510,510) \rightarrow  \mathrm{EnBlock}(510,510) \rightarrow  \\
    & \mathrm{FC}(510,255) \rightarrow \mathrm{MaxPool}  \stackrel{\times 2}{\rightarrow} \mathrm{Linear}(255,255),
    \end{align*}
    where the final two linear layers outputs vectors $(\vmu,\veta)$, with $\vmu\in\Real^{255}$ is the latent mean, and $\veta\in\Real^{255}$ is the latent log-standard-deviation.
    
    The decoder, $\psi$, similarly to \cite{atzmon2020sald,park2019deepsdf}, is a composition of $8$ layers where the first layer is $\mathrm{FC}(255 + 3,512)$, the second and third layer are $\mathrm{FC}(512,512)$, the fourth layer is a concatenation of $\mathrm{FC}(512,255)$ to the input, the fifth to seventh layers are $\mathrm{FC}(512,512)$, and the final layer is $\mathrm{Linear}(512,1)$.
    For the Frame Averaging architecture of $\Phi$ and $\Psi$, we use the same as the above backbone encoder $\phi$ and decoder $\psi$. We set $m=0$ and $d=85$. For the frames construction of $\Phi$ and $\Psi$, $\gF:V\too 2^{E(3)}\setminus \emptyset$, we use the construction defined in section \ref{ss:frame_averaging}, with $\vw=\one$. Note that for $\Phi$ the frames are calculated with respect to the input point cloud $\mV = \mX \in \Real ^{6400 \times 3}$, and for $\Psi$ the frames are calculated with respect to $\mV = \mZ \in \Real^{85 \times 3}$.
    
    \paragraph{Implementation details.}
    For all methods in table \ref{tab:common3D}, training was done using the \textsc{Adam} \cite{kingma2014adam} optimizer, with batch size $64$. The number of epochs was set to $4000$, with $0.0001$ fixed learning rate. Training was done on $4$ Nvidia QUADRO RTX 8000 GPUs.
    
    \paragraph{Evaluation.}
    We used the following Chamfer distance metrics to measure similarity between shapes:
    \begin{align} \label{e:CD}
    \dist_{\text{C}}\left(\gX_{1},\gX_{2} \right) & = \frac{1}{2}\left(\dist_{\text{C}}^{\rightarrow}\left(\gX_{1},\gX_{2} \right) + \dist_{\text{C}}^{\rightarrow}\left(\gX_{2},\gX_{1} \right) \right)
    \end{align}
    where 
    \begin{align}\label{e:one_sided_CD}
    \dist_{\text{C}}^{\rightarrow}\left(\gX_{1},\gX_{2} \right) & = \frac{1}{\abs{\gX_1}}\sum_{\vx_{1}\in\gX_{1}}\min_{\vx_2\in \gX_{2}}\norm{\vx_1-\vx_2}^2
    \end{align}
    and the sets $\gX_{i}$ are point clouds. As the input test set are given as 3D point clouds, we use the test point clouds directly with this metric. However, as the reconstructions are represented as implicit surfaces, we first mesh the predictions zero levelset using the $\mathrm{Marching Cubes}$ \cite{lorensen1987marching} algorithm, and second, we generate a point cloud of $30000$ points by uniformly sample from the predicted mesh. 
    
    \subsection{Piecewise Euclidean: Mesh $\too$ mesh}
    \paragraph{Architecture.}
    Here we provide additional details regarding the architectures for the methods in table \ref{tab:normal} applied on the DFaust splits. The AE architecture is the same as the one detailed in section \ref{sss:arch_global_mesh} for the global Euclidean mesh $\too$ mesh. For ARAPReg \cite{huang2021arapreg}, random split, we used the trained models provided with the official implementation. Note that ARAPReg is an auto-decoder, where in test time, latent codes are optimized to reconstruct test inputs. For the Frame Averaging version with the above encoder and decoder backbones, we describe next the design choices that have been made. We set the weights $\mW \in \Real ^{6890 \times 24}$ according to the template skinning weights provided with SMPL \cite{SMPL:2015}. For the random split, the latent space dimensions were set to $k=24$, $m=12$, $d=20$. For both the unseen poses splits the latent space dimensions were set to $k=24$, $m=72$, $d=24$. For the frames of $\Phi$,  $\gF_i:V\too 2^{E(3)}\setminus \emptyset$, with $V=\Real^{6890 \times 3}$ and $1 \leq i \leq 24$, we use the construction defined in section \ref{ss:frame_averaging}. We set the weights $\vw$ used for the weighted PCA of the $i$-th frame with  $\mW \ve_i$, where $\mW$ is the template skinning weights and $\ve_i$ is the $i$-th element in the standard basis of $\Real^{24}$. The frames of $\ip{\psi}_{\gF}$, $\gF:V\too 2^{E(3)}\setminus \emptyset$ (where $V=\Real^{20 \times 3}$), are calculated with $\vw=\one$. Note that exactly the same Frame Averaging architecture was also used in the experiment in table \ref{tab:all_humans_splits}.
    
    For the experiment with SMAL in table \ref{tab:normal}, we next provide the baseline AE architecture based on the one provided in ARAPReg. The encoder, $\phi$, is consisted of the following blocks, where each block is defined as in section \ref{sss:arch_global_mesh}:
    \begin{align*}
    &\mathrm{EnBlock}(3889,1945,3,32) \rightarrow \\
    &\mathrm{EnBlock}(1945,973,32,32)  \rightarrow \\
    &\mathrm{EnBlock}(973,973,32,32) \rightarrow \\
    &\mathrm{EnBlock}(973,973,32,64) \rightarrow \mathrm{Linear}(62272,96)  \rightarrow.    
    \end{align*}
    The decoder, $\psi$, consists of the following blocks:
    \begin{align*}
    &\mathrm{Linear}(96,62272) \rightarrow \mathrm{DeBlock}(973,973,64,64)  \rightarrow \\
    &\mathrm{DeBlock}(973,973,64,32) \rightarrow \\
    &\mathrm{DeBlock}(973,1945,32,32) \rightarrow \\
    &\mathrm{DeBlock}(1945,3889,32,32) \rightarrow \\
    &\mathrm{Conv}(3889,3889,32,3) \rightarrow.    
    \end{align*}
    For our Frame Averaging architecture we set $k=33$, $m=12$, $d=28$. The weights $\mW \in \Real^{3889 \times 33}$ are set according to the template skinning weights provided with the dataset. The frames for the encoder and the decoder are constructed in the same fashion as the one described above for the Frame Averaging DFaust architecture. 
    \paragraph{Implementation details.}
    For our Frame Averaging method in tables \ref{tab:normal} and \ref{tab:all_humans_splits}, training was done using the \textsc{Adam} \cite{kingma2014adam} optimizer, with batch size $16$. The number of epochs was set to $100$, with $0.0001$ fixed learning rate. Training was done on $4$ Nvidia QUADRO RTX 8000 GPUs.
    For ARAPReg, the numbers reported for the random split and SMAL in table \ref{tab:normal} are from the original paper \cite{huang2021arapreg}. For the unseen pose and unseen global pose splits in table \ref{tab:normal}, we retrained the ARAPReg and AE using the official implementation training details, where the only change is that the latent size was set to $144$ from $72$ (to be equal to ours).
    The numbers reported in table \ref{tab:all_humans_splits} for the baselines SNARF\cite{chen2021snarf} and NASA \cite{deng2020nasa} are from table 1 in \cite{chen2021snarf}. The qualitative results for SNARF in figure \ref{fig:geiger} were produced using SNARF official implementation. Note that the meshes for the qualitative results in the original SNARF paper were extracted using a \emph{Multiresolution IsoSurface Extraction} \cite{mescheder2019occupancy}, whereas the meshing in the official SNARF implementation (used in this paper) is based on the standard $\mathrm{MarchingCubes}$ \cite{lorensen1987marching} algorithm, result in with more severe staircase artifacts. 
    
    \paragraph{Evaluation.}
    The definition for the MSE metric reported in \ref{tab:normal} is the same as defined above in equation \ref{eq:mse}. The numbers reported for our method in table \ref{tab:normal}, were produced with the evaluation code from the official implementation of SNARF. For completeness, we repeat the procedure for the calculation of the IoU metric here: Let $\mX$ be a ground truth shape and $\mY$ is the corresponding prediction. Then, two samples of points in space, bbox (bounding box) and near surface, are constructed. The near surface sample is constructed by adding a Gaussian noise with $\sigma=0.01$ to a sample of $200000$ points from the shape $\mX$. The bbox is constructed by uniformly sampling of $100000$ points in a the bounding box of all shapes in the dataset. Then, the IoU is computed by 
    $$
        \text{IoU} = \frac{1}{\vert S \vert} \sum _{\vx \in S} o(\vx;\mX) * o(\vx;\mY) 
    $$
    where $S$ is the prepared sample (bbox or near surface), and $o(\cdot;\mX) \in \left\{0,1\right\}$, $o(\cdot;\mY) \in \left\{0,1\right\}$ are the occupancy functions of the shapes $\mX$ and $\mY$.
    
    
    \subsection{Interpolation in shape space}
    Here we provide additional details regarding the interpolation in shape space. Let $\mZ^{(j)}=(\vq^{(j)},\mQ^{(j)})\in Z$, $j=0,1$, be two latent codes. Here we describe the case $k=1$, as for $k>1$ the following scheme can be applied in the same manner for each part. Note that $\vq^{(j)}$ is the invariant features, whereas $\mQ^{(j)}$ is the equivariant features. Let $\mR$ be the optimal rotation between $\hat{\mQ}^{(0)}$ to $\hat{\mQ}^{(1)}$, where $\hat{\mQ}^{(j)}$ denotes the centered version of $\mQ^{(j)}$. Then,
    \begin{align*}
    \mQ_{t} = &\mathrm{slerp}(t,\mI,\mR)(\hat{\mQ}^{(0)} + t \mD) + \\
    &t \frac{1}{d} \one \one^T \mQ^{(0)} + (1-t) \frac{1}{d} \one \one^T \mQ^{(1)}
    \end{align*}
    where $\mD = \mR^T \mQ^{(1)} - \mQ^{(0)}$, and $\mathrm{slerp}(t,\mI,\mR)$ denotes the spherical liner interpolation between $\mI$ to $\mR$.
    For the invariant features, we perform regular linear interpolation $\vq_t = t \vq^{(0)} + (1-t)\vq^{(1)} $. Then, the interpolant at time $t$ is 
    $$
    \mZ_t = (\vq_t,\mQ_t).
    $$

    \subsection{Timings}
    \begin{figure}[t]
        \centering
        \includegraphics[width=\columnwidth]{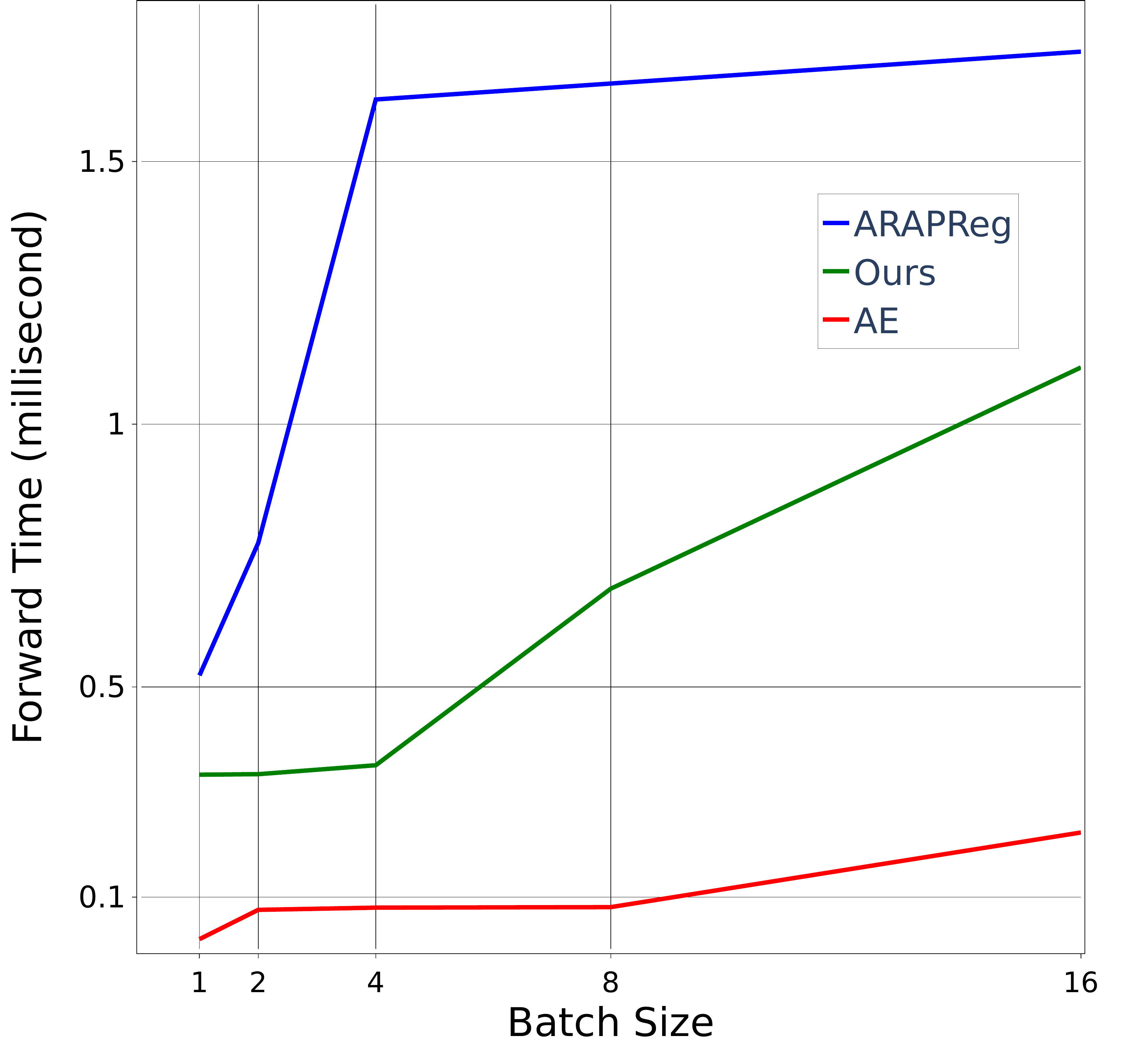}
        \caption{Timings for the methods in table \ref{tab:normal} for various batch sizes. \vspace{0pt}}
        \label{fig:timings}
        
    \end{figure}
    
    In figure \ref{fig:timings} we report the training forward time for various batch sizes for the methods in table \ref{tab:normal}. Note that the increase in forward time for our method with respect to the baseline AE, stems from the computation of the backbone network for all the different parts and possible elements in the frame. Note that for ARAPReg, the forward time in test reduces significantly as the computation of the ARAP regularizer is not needed. However, note that in the case of ARAReg with an AutoDecoder, prediction time increases significantly as it requires optimizing the latent codes.

    \section{Proofs}
    
    \begin{proof}[Proof of Proposition \ref{prop:frame_equi}]
    The proof basically repeats the proof in \cite{puny2021frame} but for the weighted case. Let $g = (\mS,\vu)$ be an element in $E(3)$. We need to show that $\gF(\rho_1(g) \mV) = g \gF(\mV)$. The above equality for the translations part $\Real^3$  in the group $E(3) = O(3)\ltimes \Real^3$ is trivial. For showing the above equality for $O(3)$, Let $\mC_{\mV}$ be the covariance matrix of $\mV$. That is $\mC_{\mV} =  (\mV-\frac{\one \vw^T }{\one ^T \vw}\mV)^T\mathrm{diag}(\vw)(\mV-\frac{\one \vw^T}{\one ^T \vw}\mV)$. A direct calculation shows that $$\mC_{\mV} = \mV^T \left[ \mathrm{diag}(\vw) \left(\mI - \frac{\one \vw ^T}{\one ^T \vw} \right) \right] \mV. $$ Then, the covariance matrix of $\rho_1(g)\mV$ satisfies 
    $$
    \mC_{\rho_1(g) \mV} = \mS^T \mV^T \left[ \mathrm{diag}(\vw) \left(\mI - \frac{\one \vw ^T}{\one ^T \vw} \right) \right] \mV \mS.
    $$
    Thus, $\vr$ is an eigen vector of $\mC_{\mV}$ if and only if $\mS\vr$ is an eigen vector of $\mC_{\rho_1(g) \mV}$.
    To finish the proof, notice that by definition $g \gF(\mV) = \set{(\mS \mR,\mS \vt + \vu)\ \vert \ \mR=\brac{\pm\vr_1,\pm\vr_2,\pm\vr_3}}$.
    
    \end{proof}
    
    \begin{proof}[Proof of Proposition \ref{lem:Psi_hat_Psi}] We need to show that 
    \begin{align*}
        \Psi(\rho_Z(g)\mZ) = \rho_Y(g)\Psi(\mZ)
    \end{align*}
    which is equivalent to showing that for all $\vx\in\Real^3$
    \begin{align*}
        \hat{\Psi}(\rho_Z(g)\mZ,\vx)=\hat{\Psi}(\mZ,\mR^T(\vx-\vt))
    \end{align*}
    which is again equivalent to showing that for all $\vx\in\Real^3$
    \begin{align*}
        \hat{\Psi}(\rho_Z(g)\mZ,\mR\vx+\vt)=\hat{\Psi}(\mZ,\vx)
    \end{align*}
    which by definition of $\rho_V$ is finally equivalent to showing for all $\vx\in\Real^3$ 
    \begin{align*}
        \hat{\Psi}(\rho_P(g)(\mZ,\vx))= \rho_\Real(g)\hat{\Psi}(\mZ,\vx)
    \end{align*}
    as required.
    \end{proof}
    
    \begin{proof}[Proof of Theorem \ref{thm:part_equi}.]
    If $g\in \gF_j(\mX)$ then $g=(\mR,\vt)$, where $\vt=(\one^T \vw_j)^{-1}\mX^T\vw_j$ as defined in \eqref{e:t}. Therefore
    \begin{align*}
        \rho_X(g)^{-1}\mX_j &= ((\one-\vw_j) \vt^T+\vw_j\odot\mX - \one\vt^T)\mR \\
        &= \vw_j \odot (\mX-\one \vt^T)\mR \\
        &= \vw_j \odot \rho_X(g)^{-1}\mX
    \end{align*}
    Furthermore, it can be directly checked that for hard weights $\gF_j(\mX)$ is an equivariant frame that is defined only in terms of the rows of $\mX$ that belong to the $j$-th part. Now, $\mZ_j=\ip{\phi}_{\gF_j}(\mX_j)$ where 
    \begin{align*}
        &\ip{\phi}_{\gF_j}(\mX_j) = \frac{1}{|\gF_j(\mX)|}\sum_{g\in \gF_j(\mX)}\rho_{Z_j}(g)\phi(\rho_X(g)^{-1}\mX_j) \\
        &= \frac{1}{|\gF_j(\mX)|}\sum_{g\in \gF_j(\mX)}\rho_{Z_j}(g)\phi(\vw_j\odot \rho_X(g)^{-1}\mX).
    \end{align*}
    Therefore Frame Averaging now implies that $\ip{\mX_j}_{\gF}$ is equivariant to Euclidean motions of the $j$-th part. For the decoder, $\ip{\mZ_j}_{\gF}$ is equivariant to Euclidean motions of $\mZ_j$ by Frame Averaging, and $\vw_j\odot \rho_Y(g_j)\mY$ transforms the decoded $j$-th part accordingly. 
    \end{proof}
    
    \section{Additional results}
    
    We provide additional qualitative results for some of the experiments in section \ref{sec:exp}. Figure \ref{fig:geiger_more} shows additional qualitative results for the comparison with an implicit pose-conditioned method (SNARF).
    Figure \ref{fig:smal} shows typical test reconstructions on the SMAL dataset. Note that the test set here does not include out of distribution poses, hence relatively easy. Note that our method is able to capture the pose better than the baselines (see the animal's tail for example), while achieving high fidelity (see base of the feet). 
    
    \begin{figure*}[h]
        \centering
        \includegraphics[width=1.0\textwidth]{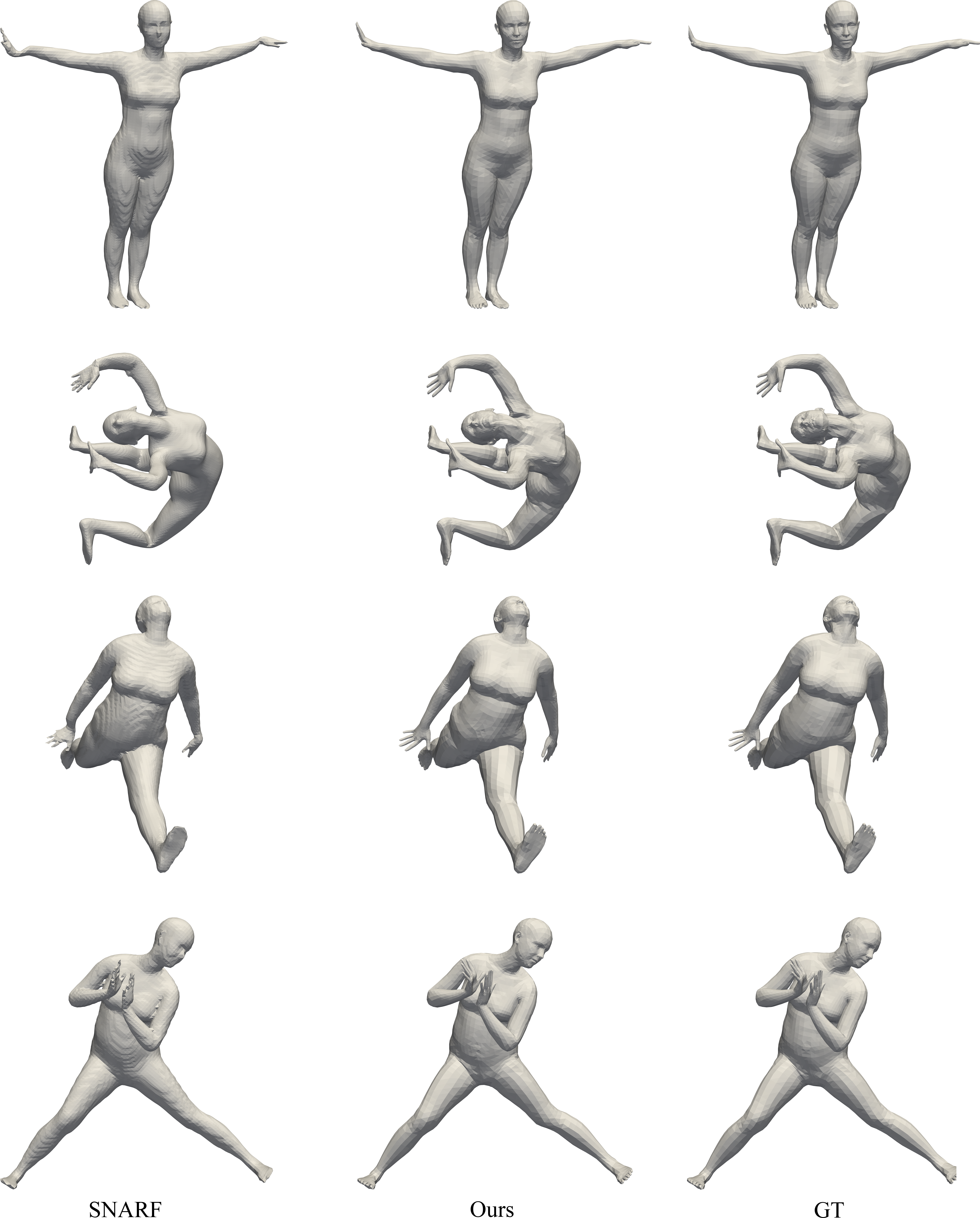}
        \caption{Comparison to SNARF on "out of distribution" test.}
        \label{fig:geiger_more}
    \end{figure*}

    \begin{figure*}[h]
        \centering
        \includegraphics[width=1.0\textwidth]{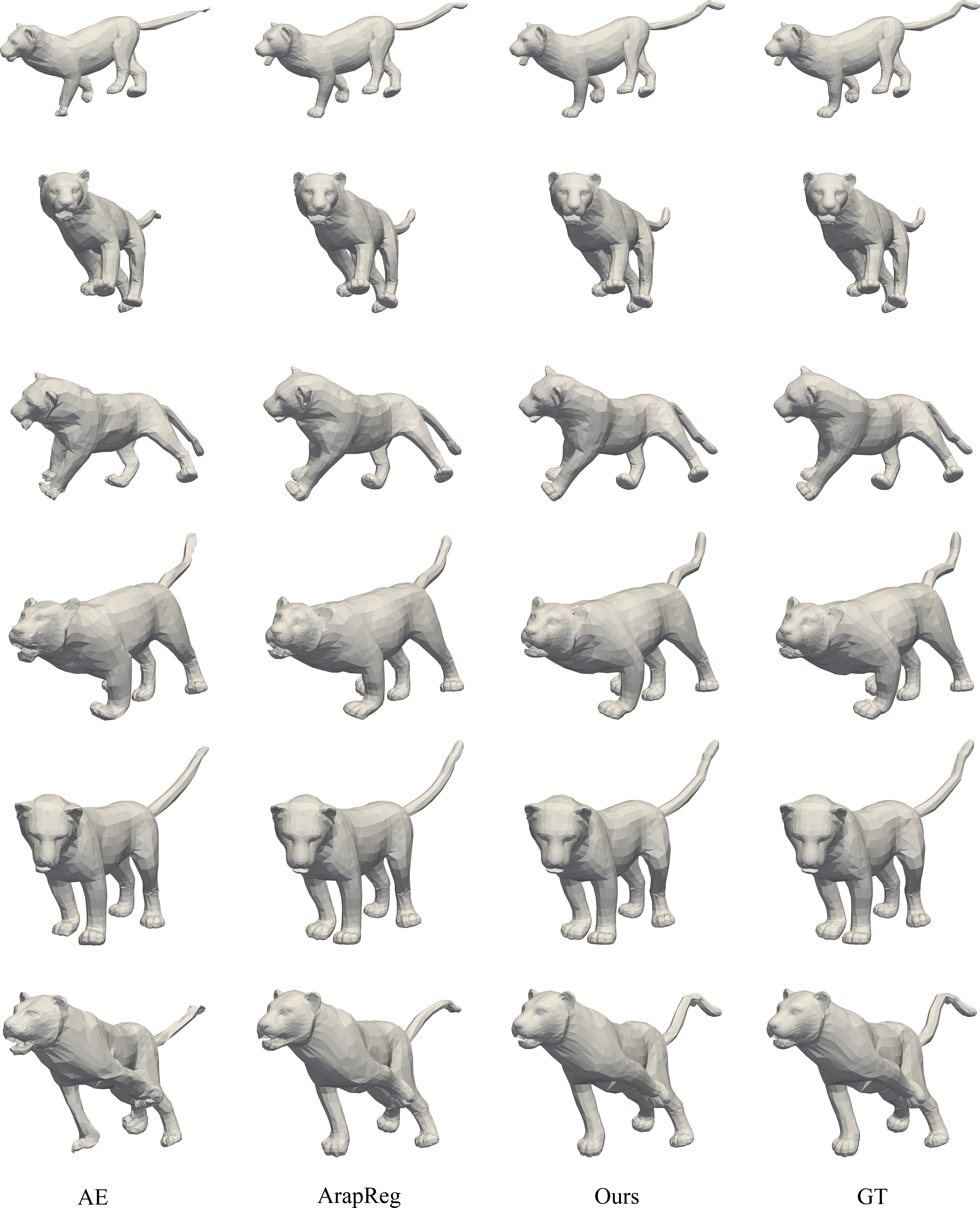}
        \caption{Piecewise Euclidean mesh $\too$ mesh, qualitative results; SMAL \cite{zuffi20173d} dataset.}
        \label{fig:smal}
    \end{figure*}

\end{document}

%% file: math_commands.tex
\usepackage{amsmath,amsfonts,bm}
\usepackage{multirow}
\usepackage{tabularx}
\usepackage{graphicx}
\usepackage{adjustbox}
\usepackage{amsthm}

\renewcommand{\arraystretch}{1.3}
\newcommand{\mset}[1]{\left\{\kern-.5em\left\{ #1 \right\}\kern-.5em\right\}}
\newcommand{\mmset}[1]{\{\kern-.4em\{ #1 \}\kern-.4em\}}

\newcommand{\vertiii}[1]{{\left\vert\kern-0.25ex\left\vert\kern-0.25ex\left\vert #1 
    \right\vert\kern-0.25ex\right\vert\kern-0.25ex\right\vert}}

\newcommand{\norm}[1]{\left\Vert#1\right\Vert}

\newcommand{\abs}[1]{\left\vert#1\right\vert}

\newcommand{\set}[1]{\left\{#1\right\}}
\newcommand{\parr}[1]{\left (#1\right )}
\newcommand{\brac}[1]{\left [#1\right ]}

\newcommand{\ip}[1]{\left \langle #1 \right \rangle }
\newcommand{\Real}{\mathbb R}
\newcommand{\Nat}{\mathbb N}

\newcommand{\too}{\rightarrow}

\newcommand{\diag}{\textrm{diag}} 

\newcommand{\one}{\mathbf{1}}

\newcommand{\eg}{{e.g.}}
\newcommand{\ie}{{i.e.}}

 \newtheorem{theorem}{Theorem}
 
 \newtheorem{proposition}{Proposition}

\def\eqref#1{equation~\ref{#1}}

\def\1{\bm{1}}

\def\vmu{{\bm{\mu}}}

\def\veta{{\bm{\eta}}}

\def\va{{\bm{a}}}
\def\vb{{\bm{b}}}

\def\ve{{\bm{e}}}

\def\vq{{\bm{q}}}
\def\vr{{\bm{r}}}

\def\vt{{\bm{t}}}
\def\vu{{\bm{u}}}

\def\vw{{\bm{w}}}
\def\vx{{\bm{x}}}

\def\vec1{{\bm{1}}}

\def\mB{{\bm{B}}}
\def\mC{{\bm{C}}}
\def\mD{{\bm{D}}}

\def\mI{{\bm{I}}}

\def\mQ{{\bm{Q}}}
\def\mR{{\bm{R}}}
\def\mS{{\bm{S}}}

\def\mU{{\bm{U}}}
\def\mV{{\bm{V}}}
\def\mW{{\bm{W}}}
\def\mX{{\bm{X}}}
\def\mY{{\bm{Y}}}
\def\mZ{{\bm{Z}}}

\DeclareMathAlphabet{\mathsfit}{\encodingdefault}{\sfdefault}{m}{sl}
\SetMathAlphabet{\mathsfit}{bold}{\encodingdefault}{\sfdefault}{bx}{n}

\def\gF{{\mathcal{F}}}

\def\gL{{\mathcal{L}}}

\def\gN{{\mathcal{N}}}

\def\gX{{\mathcal{X}}}

\newcommand{\dist}{\textrm{d}} 



%% file: paper_arxiv_v2.bbl
\begin{thebibliography}{10}\itemsep=-1pt

\bibitem{akhter2015pose}
Ijaz Akhter and Michael~J Black.
\newblock Pose-conditioned joint angle limits for 3d human pose reconstruction.
\newblock In {\em Proceedings of the IEEE conference on computer vision and
  pattern recognition}, pages 1446--1455, 2015.

\bibitem{atzmon2020sal}
Matan Atzmon and Yaron Lipman.
\newblock Sal: Sign agnostic learning of shapes from raw data.
\newblock In {\em Proceedings of the IEEE/CVF Conference on Computer Vision and
  Pattern Recognition}, pages 2565--2574, 2020.

\bibitem{atzmon2020sald}
Matan Atzmon and Yaron Lipman.
\newblock {SALD:} sign agnostic learning with derivatives.
\newblock In {\em 9th International Conference on Learning Representations,
  {ICLR} 2021}, 2021.

\bibitem{atzmon2018point}
Matan Atzmon, Haggai Maron, and Yaron Lipman.
\newblock Point convolutional neural networks by extension operators.
\newblock {\em ACM Transactions on Graphics (TOG)}, 37(4):1--12, 2018.

\bibitem{atzmon2021augmenting}
Matan Atzmon, David Novotny, Andrea Vedaldi, and Yaron Lipman.
\newblock Augmenting implicit neural shape representations with explicit
  deformation fields.
\newblock {\em arXiv preprint arXiv:2108.08931}, 2021.

\bibitem{dfaust:CVPR:2017}
Federica Bogo, Javier Romero, Gerard Pons-Moll, and Michael~J. Black.
\newblock Dynamic {FAUST}: {R}egistering human bodies in motion.
\newblock In {\em IEEE Conf. on Computer Vision and Pattern Recognition
  (CVPR)}, July 2017.

\bibitem{breiding2018geometry}
Paul Breiding, Khazhgali Kozhasov, and Antonio Lerario.
\newblock On the geometry of the set of symmetric matrices with repeated
  eigenvalues.
\newblock {\em Arnold Mathematical Journal}, 4(3):423--443, 2018.

\bibitem{chen2021equivariant}
Haiwei Chen, Shichen Liu, Weikai Chen, Hao Li, and Randall Hill.
\newblock Equivariant point network for 3d point cloud analysis.
\newblock In {\em Proceedings of the IEEE/CVF Conference on Computer Vision and
  Pattern Recognition}, pages 14514--14523, 2021.

\bibitem{chen2021snarf}
Xu Chen, Yufeng Zheng, Michael~J Black, Otmar Hilliges, and Andreas Geiger.
\newblock Snarf: Differentiable forward skinning for animating non-rigid neural
  implicit shapes.
\newblock {\em arXiv preprint arXiv:2104.03953}, 2021.

\bibitem{chen2019learning}
Zhiqin Chen and Hao Zhang.
\newblock Learning implicit fields for generative shape modeling.
\newblock In {\em Proceedings of the IEEE/CVF Conference on Computer Vision and
  Pattern Recognition}, pages 5939--5948, 2019.

\bibitem{chibane2020implicit}
Julian Chibane, Thiemo Alldieck, and Gerard Pons-Moll.
\newblock Implicit functions in feature space for 3d shape reconstruction and
  completion.
\newblock In {\em Proceedings of the IEEE/CVF Conference on Computer Vision and
  Pattern Recognition}, pages 6970--6981, 2020.

\bibitem{cohen2018spherical}
Taco~S Cohen, Mario Geiger, Jonas K{\"o}hler, and Max Welling.
\newblock Spherical cnns.
\newblock In {\em International Conference on Learning Representations}, 2018.

\bibitem{cohen2016steerable}
Taco~S Cohen and Max Welling.
\newblock Steerable cnns.
\newblock {\em arXiv preprint arXiv:1612.08498}, 2016.

\bibitem{defferrard2016convolutional}
Micha{\"e}l Defferrard, Xavier Bresson, and Pierre Vandergheynst.
\newblock Convolutional neural networks on graphs with fast localized spectral
  filtering.
\newblock {\em Advances in neural information processing systems},
  29:3844--3852, 2016.

\bibitem{deng2020nasa}
Boyang Deng, John~P Lewis, Timothy Jeruzalski, Gerard Pons-Moll, Geoffrey
  Hinton, Mohammad Norouzi, and Andrea Tagliasacchi.
\newblock Nasa neural articulated shape approximation.
\newblock In {\em Computer Vision--ECCV 2020: 16th European Conference,
  Glasgow, UK, August 23--28, 2020, Proceedings, Part VII 16}, pages 612--628.
  Springer, 2020.

\bibitem{deng2021vector}
Congyue Deng, Or Litany, Yueqi Duan, Adrien Poulenard, Andrea Tagliasacchi, and
  Leonidas~J Guibas.
\newblock Vector neurons: A general framework for so (3)-equivariant networks.
\newblock In {\em Proceedings of the IEEE/CVF International Conference on
  Computer Vision}, pages 12200--12209, 2021.

\bibitem{deng2018ppf}
Haowen Deng, Tolga Birdal, and Slobodan Ilic.
\newblock Ppf-foldnet: Unsupervised learning of rotation invariant 3d local
  descriptors.
\newblock In {\em Proceedings of the European Conference on Computer Vision
  (ECCV)}, pages 602--618, 2018.

\bibitem{dym2020universality}
Nadav Dym and Haggai Maron.
\newblock On the universality of rotation equivariant point cloud networks.
\newblock In {\em International Conference on Learning Representations}, 2020.

\bibitem{esteves2018learning}
Carlos Esteves, Christine Allen-Blanchette, Ameesh Makadia, and Kostas
  Daniilidis.
\newblock Learning so (3) equivariant representations with spherical cnns.
\newblock In {\em Proceedings of the European Conference on Computer Vision
  (ECCV)}, pages 52--68, 2018.

\bibitem{fuchs2020se}
Fabian Fuchs, Daniel Worrall, Volker Fischer, and Max Welling.
\newblock Se (3)-transformers: 3d roto-translation equivariant attention
  networks.
\newblock {\em Advances in Neural Information Processing Systems}, 33, 2020.

\bibitem{gojcic2019perfect}
Zan Gojcic, Caifa Zhou, Jan~D Wegner, and Andreas Wieser.
\newblock The perfect match: 3d point cloud matching with smoothed densities.
\newblock In {\em Proceedings of the IEEE/CVF Conference on Computer Vision and
  Pattern Recognition}, pages 5545--5554, 2019.

\bibitem{gropp2020implicit}
Amos Gropp, Lior Yariv, Niv Haim, Matan Atzmon, and Yaron Lipman.
\newblock Implicit geometric regularization for learning shapes.
\newblock In {\em Proceedings of Machine Learning and Systems 2020}, pages
  3569--3579. 2020.

\bibitem{huang2021arapreg}
Qixing Huang, Xiangru Huang, Bo Sun, Zaiwei Zhang, Junfeng Jiang, and
  Chandrajit Bajaj.
\newblock Arapreg: An as-rigid-as possible regularization loss for learning
  deformable shape generators.
\newblock In {\em Proceedings of the IEEE/CVF International Conference on
  Computer Vision}, pages 5815--5825, 2021.

\bibitem{jeruzalski2020nilbs}
Timothy Jeruzalski, David~IW Levin, Alec Jacobson, Paul Lalonde, Mohammad
  Norouzi, and Andrea Tagliasacchi.
\newblock Nilbs: Neural inverse linear blend skinning.
\newblock {\em arXiv preprint arXiv:2004.05980}, 2020.

\bibitem{jiang2020disentangled}
Boyi Jiang, Juyong Zhang, Jianfei Cai, and Jianmin Zheng.
\newblock Disentangled human body embedding based on deep hierarchical neural
  network.
\newblock {\em IEEE transactions on visualization and computer graphics},
  26(8):2560--2575, 2020.

\bibitem{kingma2014adam}
Diederik~P Kingma and Jimmy Ba.
\newblock Adam: A method for stochastic optimization.
\newblock {\em arXiv preprint arXiv:1412.6980}, 2014.

\bibitem{kostrikov2018surface}
Ilya Kostrikov, Zhongshi Jiang, Daniele Panozzo, Denis Zorin, and Joan Bruna.
\newblock Surface networks.
\newblock In {\em Proceedings of the IEEE Conference on Computer Vision and
  Pattern Recognition}, pages 2540--2548, 2018.

\bibitem{li2018pointcnn}
Yangyan Li, Rui Bu, Mingchao Sun, Wei Wu, Xinhan Di, and Baoquan Chen.
\newblock Pointcnn: Convolution on x-transformed points.
\newblock {\em Advances in neural information processing systems}, 31:820--830,
  2018.

\bibitem{lipman2021phase}
Yaron Lipman.
\newblock Phase transitions, distance functions, and implicit neural
  representations.
\newblock {\em arXiv preprint arXiv:2106.07689}, 2021.

\bibitem{litany2018deformable}
Or Litany, Alex Bronstein, Michael Bronstein, and Ameesh Makadia.
\newblock Deformable shape completion with graph convolutional autoencoders.
\newblock In {\em Proceedings of the IEEE conference on computer vision and
  pattern recognition}, pages 1886--1895, 2018.

\bibitem{liu2018deep}
Min Liu, Fupin Yao, Chiho Choi, Ayan Sinha, and Karthik Ramani.
\newblock Deep learning 3d shapes using alt-az anisotropic 2-sphere
  convolution.
\newblock In {\em International Conference on Learning Representations}, 2018.

\bibitem{SMPL:2015}
Matthew Loper, Naureen Mahmood, Javier Romero, Gerard Pons-Moll, and Michael~J.
  Black.
\newblock {SMPL}: A skinned multi-person linear model.
\newblock {\em ACM Trans. Graphics (Proc. SIGGRAPH Asia)}, 34(6):248:1--248:16,
  Oct. 2015.

\bibitem{lorensen1987marching}
William~E Lorensen and Harvey~E Cline.
\newblock Marching cubes: A high resolution 3d surface construction algorithm.
\newblock In {\em ACM siggraph computer graphics}, volume~21, pages 163--169.
  ACM, 1987.

\bibitem{mahmood2019amass}
Naureen Mahmood, Nima Ghorbani, Nikolaus~F Troje, Gerard Pons-Moll, and
  Michael~J Black.
\newblock Amass: Archive of motion capture as surface shapes.
\newblock In {\em Proceedings of the IEEE/CVF International Conference on
  Computer Vision}, pages 5442--5451, 2019.

\bibitem{mescheder2019occupancy}
Lars Mescheder, Michael Oechsle, Michael Niemeyer, Sebastian Nowozin, and
  Andreas Geiger.
\newblock Occupancy networks: Learning 3d reconstruction in function space.
\newblock In {\em Proceedings of the IEEE/CVF Conference on Computer Vision and
  Pattern Recognition}, pages 4460--4470, 2019.

\bibitem{mihajlovic2021leap}
Marko Mihajlovic, Yan Zhang, Michael~J Black, and Siyu Tang.
\newblock Leap: Learning articulated occupancy of people.
\newblock In {\em Proceedings of the IEEE/CVF Conference on Computer Vision and
  Pattern Recognition}, pages 10461--10471, 2021.

\bibitem{niemeyer2019occupancy}
Michael Niemeyer, Lars Mescheder, Michael Oechsle, and Andreas Geiger.
\newblock Occupancy flow: 4d reconstruction by learning particle dynamics.
\newblock In {\em Proceedings of the IEEE/CVF International Conference on
  Computer Vision}, pages 5379--5389, 2019.

\bibitem{park2019deepsdf}
Jeong~Joon Park, Peter Florence, Julian Straub, Richard Newcombe, and Steven
  Lovegrove.
\newblock Deepsdf: Learning continuous signed distance functions for shape
  representation.
\newblock In {\em Proceedings of the IEEE/CVF Conference on Computer Vision and
  Pattern Recognition}, pages 165--174, 2019.

\bibitem{park2020deformable}
Keunhong Park, Utkarsh Sinha, Jonathan~T Barron, Sofien Bouaziz, Dan~B Goldman,
  Steven~M Seitz, and Ricardo Martin-Brualla.
\newblock Deformable neural radiance fields.
\newblock {\em arXiv preprint arXiv:2011.12948}, 2020.

\bibitem{peng2020convolutional}
Songyou Peng, Michael Niemeyer, Lars Mescheder, Marc Pollefeys, and Andreas
  Geiger.
\newblock Convolutional occupancy networks.
\newblock In {\em Computer Vision--ECCV 2020: 16th European Conference,
  Glasgow, UK, August 23--28, 2020, Proceedings, Part III 16}, pages 523--540.
  Springer, 2020.

\bibitem{pumarola2021d}
Albert Pumarola, Enric Corona, Gerard Pons-Moll, and Francesc Moreno-Noguer.
\newblock D-nerf: Neural radiance fields for dynamic scenes.
\newblock In {\em Proceedings of the IEEE/CVF Conference on Computer Vision and
  Pattern Recognition}, pages 10318--10327, 2021.

\bibitem{puny2021frame}
Omri Puny, Matan Atzmon, Edward~J Smith, Ishan Misra, Aditya Grover, Heli
  Ben-Hamu, and Yaron Lipman.
\newblock Frame averaging for invariant and equivariant network design.
\newblock In {\em International Conference on Learning Representations}, 2021.

\bibitem{qi2017pointnet}
Charles~R Qi, Hao Su, Kaichun Mo, and Leonidas~J Guibas.
\newblock Pointnet: Deep learning on point sets for 3d classification and
  segmentation.
\newblock In {\em Proceedings of the IEEE conference on computer vision and
  pattern recognition}, pages 652--660, 2017.

\bibitem{qi2017pointnet++}
Charles~R Qi, Li Yi, Hao Su, and Leonidas~J Guibas.
\newblock Pointnet++: Deep hierarchical feature learning on point sets in a
  metric space.
\newblock {\em arXiv preprint arXiv:1706.02413}, 2017.

\bibitem{ranjan2018generating}
Anurag Ranjan, Timo Bolkart, Soubhik Sanyal, and Michael~J Black.
\newblock Generating 3d faces using convolutional mesh autoencoders.
\newblock In {\em Proceedings of the European Conference on Computer Vision
  (ECCV)}, pages 704--720, 2018.

\bibitem{reizenstein2021common}
Jeremy Reizenstein, Roman Shapovalov, Philipp Henzler, Luca Sbordone, Patrick
  Labatut, and David Novotny.
\newblock Common objects in 3d: Large-scale learning and evaluation of
  real-life 3d category reconstruction.
\newblock In {\em Proceedings of the IEEE/CVF International Conference on
  Computer Vision}, pages 10901--10911, 2021.

\bibitem{romero2020group}
David~W Romero and Jean-Baptiste Cordonnier.
\newblock Group equivariant stand-alone self-attention for vision.
\newblock In {\em ICLR}, 2021.

\bibitem{Saito:CVPR:2021}
Shunsuke Saito, Jinlong Yang, Qianli Ma, and Michael~J. Black.
\newblock {SCANimate}: Weakly supervised learning of skinned clothed avatar
  networks.
\newblock In {\em Proceedings IEEE/CVF Conf.~on Computer Vision and Pattern
  Recognition (CVPR)}, June 2021.

\bibitem{schoenberger2016sfm}
Johannes~Lutz Sch\"{o}nberger and Jan-Michael Frahm.
\newblock Structure-from-motion revisited.
\newblock In {\em Conference on Computer Vision and Pattern Recognition
  (CVPR)}, 2016.

\bibitem{schonemann1966generalized}
Peter~H Sch{\"o}nemann.
\newblock A generalized solution of the orthogonal procrustes problem.
\newblock {\em Psychometrika}, 31(1):1--10, 1966.

\bibitem{sitzmann2020implicit}
Vincent Sitzmann, Julien Martel, Alexander Bergman, David Lindell, and Gordon
  Wetzstein.
\newblock Implicit neural representations with periodic activation functions.
\newblock {\em Advances in Neural Information Processing Systems}, 33, 2020.

\bibitem{sorkine2007rigid}
Olga Sorkine and Marc Alexa.
\newblock As-rigid-as-possible surface modeling.
\newblock In {\em Symposium on Geometry processing}, volume~4, pages 109--116,
  2007.

\bibitem{thomas2018tensor}
Nathaniel Thomas, Tess Smidt, Steven Kearnes, Lusann Yang, Li Li, Kai Kohlhoff,
  and Patrick Riley.
\newblock Tensor field networks: Rotation-and translation-equivariant neural
  networks for 3d point clouds.
\newblock {\em arXiv preprint arXiv:1802.08219}, 2018.

\bibitem{verma2018feastnet}
Nitika Verma, Edmond Boyer, and Jakob Verbeek.
\newblock Feastnet: Feature-steered graph convolutions for 3d shape analysis.
\newblock In {\em Proceedings of the IEEE conference on computer vision and
  pattern recognition}, pages 2598--2606, 2018.

\bibitem{wand2007reconstruction}
Michael Wand, Philipp Jenke, Qixing Huang, Martin Bokeloh, Leonidas Guibas, and
  Andreas Schilling.
\newblock Reconstruction of deforming geometry from time-varying point clouds.
\newblock In {\em Symposium on Geometry processing}, pages 49--58, 2007.

\bibitem{wang2019dynamic}
Yue Wang, Yongbin Sun, Ziwei Liu, Sanjay~E Sarma, Michael~M Bronstein, and
  Justin~M Solomon.
\newblock Dynamic graph cnn for learning on point clouds.
\newblock {\em Acm Transactions On Graphics (tog)}, 38(5):1--12, 2019.

\bibitem{weiler20183d}
Maurice Weiler, Mario Geiger, Max Welling, Wouter Boomsma, and Taco Cohen.
\newblock 3d steerable cnns: Learning rotationally equivariant features in
  volumetric data.
\newblock {\em arXiv preprint arXiv:1807.02547}, 2018.

\bibitem{worrall2017harmonic}
Daniel~E Worrall, Stephan~J Garbin, Daniyar Turmukhambetov, and Gabriel~J
  Brostow.
\newblock Harmonic networks: Deep translation and rotation equivariance.
\newblock In {\em Proceedings of the IEEE Conference on Computer Vision and
  Pattern Recognition}, pages 5028--5037, 2017.

\bibitem{xiao2020endowing}
Zelin Xiao, Hongxin Lin, Renjie Li, Lishuai Geng, Hongyang Chao, and Shengyong
  Ding.
\newblock Endowing deep 3d models with rotation invariance based on principal
  component analysis.
\newblock In {\em 2020 IEEE International Conference on Multimedia and Expo
  (ICME)}, pages 1--6. IEEE, 2020.

\bibitem{xu2018spidercnn}
Yifan Xu, Tianqi Fan, Mingye Xu, Long Zeng, and Yu Qiao.
\newblock Spidercnn: Deep learning on point sets with parameterized
  convolutional filters.
\newblock In {\em Proceedings of the European Conference on Computer Vision
  (ECCV)}, pages 87--102, 2018.

\bibitem{yarotsky2021universal}
Dmitry Yarotsky.
\newblock Universal approximations of invariant maps by neural networks.
\newblock {\em Constructive Approximation}, pages 1--68, 2021.

\bibitem{yu2020deep}
Ruixuan Yu, Xin Wei, Federico Tombari, and Jian Sun.
\newblock Deep positional and relational feature learning for
  rotation-invariant point cloud analysis.
\newblock In {\em European Conference on Computer Vision}, pages 217--233.
  Springer, 2020.

\bibitem{zaheer2017deep}
Manzil Zaheer, Satwik Kottur, Siamak Ravanbakhsh, Barnabas Poczos, Ruslan
  Salakhutdinov, and Alexander Smola.
\newblock Deep sets.
\newblock {\em arXiv preprint arXiv:1703.06114}, 2017.

\bibitem{zhang2019rotation}
Zhiyuan Zhang, Binh-Son Hua, David~W Rosen, and Sai-Kit Yeung.
\newblock Rotation invariant convolutions for 3d point clouds deep learning.
\newblock In {\em 2019 International Conference on 3D Vision (3DV)}, pages
  204--213. IEEE, 2019.

\bibitem{zhu20073d}
Yuanchen Zhu and Steven~J Gortler.
\newblock 3d deformation using moving least squares.
\newblock 2007.

\bibitem{zuffi20173d}
Silvia Zuffi, Angjoo Kanazawa, David~W Jacobs, and Michael~J Black.
\newblock 3d menagerie: Modeling the 3d shape and pose of animals.
\newblock In {\em Proceedings of the IEEE conference on computer vision and
  pattern recognition}, pages 6365--6373, 2017.

\end{thebibliography}
